\documentclass{article}

\PassOptionsToPackage{numbers,sort&compress}{natbib}

\usepackage[preprint]{neurips_2026}

\usepackage[utf8]{inputenc} 
\usepackage[T1]{fontenc}    

\usepackage{microtype}
\usepackage{graphicx}
\usepackage{enumitem}
\usepackage{booktabs} 
\usepackage{verbatim}
\usepackage{caption}
\usepackage{subcaption}
\usepackage{placeins}
\usepackage{url}

\usepackage{hyperref}

\usepackage{amsmath}
\usepackage{amssymb}
\usepackage{mathtools}
\usepackage{amsthm}
\usepackage{multirow}

\usepackage[capitalize,noabbrev]{cleveref}
\usepackage{bm}
\usepackage{xspace}
\usepackage{thmtools} 
\usepackage{thm-restate}
\theoremstyle{plain}
\usepackage{algorithm}
\usepackage{algpseudocode}
\usepackage{acronym}

\crefname{assumption}{Assumption}{Assumptions}
\crefname{lemma}{Lemma}{Lemmas}
\crefname{theorem}{Theorem}{Theorems}
\crefname{ALC@unique}{Line}{Lines}

\usepackage{xcolor} 
\usepackage[colorinlistoftodos,prependcaption,textsize=tiny]{todonotes}

\crefname{ALC@line}{line}{lines}
\crefname{ALC@unique}{line}{lines}
\Crefname{ALC@unique}{Line}{Lines}

\theoremstyle{plain}

\theoremstyle{definition}

\newtheorem{assumption}{Assumption}
\theoremstyle{remark}

\newif\ifcompactparagraph
\compactparagraphtrue

\newcommand{\parahead}[1]{%
  \ifcompactparagraph\textbf{#1}\else\paragraph{#1}\fi%
}

\newcommand{\ourmethod}{\texttt{SMOG}\xspace}
\newcommand{\priorart}{\texttt{ScaML-GP}\xspace}

\newcommand{\test}{t}
\newcommand{\nummeta}{M}
\newcommand{\metaset}{\mathcal{M}}
\newcommand{\metadata}{\mathcal{D}_{1:\nummeta}}
\newcommand{\testdata}{\mathcal{D}_\test}
\DeclareMathOperator*{\argmax}{arg\,max}

\newcommand{\zeromatrix}[1]{\bm{0}_{{#1}\times{#1}}}

\newcommand{\zerovector}[1]{\bm{0}_{1\times{#1}}}
\newcommand{\onevector}[1]{\bm{1}_{1\times{#1}}}
\newcommand{\wmatrix}[1]{\begin{pmatrix}w_{{#1}1}&w_{{#1}2}\\w_{{#1}1}&w_{{#1}2}\end{pmatrix}}

\newcommand{\wmatrixSQ}[1]{\begin{pmatrix}w_{{#1}1}^2&w_{{#1}1}w_{{#1}2}\\w_{{#1}1}w_{{#1}2}&w_{{#1}2}^2\end{pmatrix}}
\renewcommand{\epsilon}{\varepsilon}
\newcommand{\x}{\ensuremath{\mathbf{x}}}
\newcommand{\X}{\ensuremath{\mathbf{X}}}

\newcommand{\meta}{m}
\newcommand{\cov}[1]{\mathrm{Cov}\left[#1\right]}
\newcommand{\kk}[2][]{k_{#1}\!\left[#2\right]}
\newcommand{\corr}[1]{\mathrm{Corr}\left[#1\right]}

\newcommand{\diag}{\mathrm{diag}}

\newcommand{\IndGP}{\texttt{Ind.\,GP}\xspace}
\newcommand{\MOGP}{\texttt{MO-GP}\xspace}
\newcommand{\IndRGPE}{\texttt{Ind.\,RGPE}\xspace}
\newcommand{\IndSGPT}{\texttt{Ind.\,SGPT}\xspace}
\newcommand{\IndABLR}{\texttt{Ind.\,ABLR}\xspace}

\newcommand{\IndSCAML}{\texttt{Ind.\,ScaML-GP}\xspace}
\newcommand{\IndGCThreeP}{\texttt{Ind.\,GC3P}\xspace}
\newcommand{\IndSMOG}{\texttt{Ind.\,SMOG}\xspace}
\newcommand{\MOTPE}{\texttt{MO-TPE}\xspace}
\newcommand{\hartmann}{\texttt{Hartmann6}\xspace}
\newcommand{\branincurrin}{\texttt{Branin-Currin}\xspace}

\acrodef{GP}[GP]{Gaussian process}
\acrodef{BO}[BO]{Bayesian optimization}
\acrodef{MMOBO}[MMOBO]{meta-learning multi-task Bayesian optimization}
\acrodef{UAV}[UAV]{unmanned aerial vehicle}
\acrodefplural{GP}[GPs]{Gaussian processes}
\acrodef{HP}[HP]{hyperparameter}
\acrodef{HPO}[HPO]{hyperparameter optimization}
\acrodef{PSD}[PSD]{positive semi-definite}
\acrodef{MOBO}[MOBO]{multi-objective Bayesian optimization}
\acrodef{TPE}[TPE]{tree-structured Parzen estimator}
\acrodef{AF}[AF]{acquisition function}
\acrodef{HV}[HV]{hypervolume}
\acrodef{ARD}[ARD]{automated relevance determination}
\acrodef{MTGP}[MTGP]{multi-task Gaussian process}
\acrodef{RSS}[RSS]{resident set size}
\acrodef{RMSE}[RMSE]{root mean squared error}
\acrodef{NLPD}[NLPD]{negative log predictive density}
\acrodef{MLL}[MLL]{marginal log-likelihood}
\acrodef{KL}[KL]{Kullback-Leibler}
\acrodef{SEM}[SEM]{standard error of the mean}
\acrodefplural{SEM}[SEMs]{standard errors of the mean}

\title{\ourmethod: Scalable Meta-Learning for Multi-Objective Bayesian Optimization}

\author{%
  Leonard Papenmeier \\
  Department of Information Systems \\
  University of Münster, Germany \\
  \texttt{leonard.papenmeier@uni-muenster.de} \\
  \And
  Petru Tighineanu \\
  Robert Bosch GmbH \\
  Renningen, Germany \\
}

\begin{document}

\maketitle

\begin{abstract}
Multi-objective optimization aims to solve problems with competing objectives.
Evaluating such problems is often slow or expensive, limiting the budget of evaluations. 
In many applications, historical data from related optimization tasks is available and can be leveraged via \emph{meta-learning} to accelerate optimization.
\Acl{BO}, as a promising technique for expensive black-box problems, has been extended independently to meta-learning and multi-objective optimization, but methods that simultaneously address both settings remain largely unexplored.
We propose \ourmethod---a scalable and modular meta-learning model based on a multi-output \acl{GP}---that explicitly learns correlations between objectives. 
\ourmethod builds a structured joint \acl{GP} prior across meta- and target tasks and, after conditioning on metadata, yields a closed-form prior for the target task. 
This construction propagates metadata uncertainty into the target surrogate in a principled way. 
\ourmethod supports hierarchical, parallel training, achieving linear scaling with the number of meta-tasks. 
The resulting surrogate integrates seamlessly with standard \acl{MOBO} acquisition functions.
We demonstrate that our method is consistently competitive, delivering strong data efficiency across representative benchmarks and applications. 
\end{abstract}

\section{Introduction}
\label{sec:introduction}
Many high-impact optimization problems are intrinsically \emph{multi-objective}: engineers and machine-learning practitioners rarely optimize a single scalar, but rather trade off competing objectives such as performance, cost, latency, or energy. 
At the same time, these objectives are often \emph{expensive} to evaluate, naturally putting one in the low-data regime where \ac{BO}/\ac{MOBO} is particularly effective~\citep{snoek2012practical,zhang2020bayesian,BO_Shahriari,daulton2020differentiable,daulton2021parallel}. 
Organizations rarely run non-recurring optimizations; instead, they accumulate logs from past optimization runs for related products, machines, datasets, workloads, or environments. 
This makes \emph{meta-learning for expensive multi-objective optimization} an important scenario: rather than starting each new optimization from scratch, practitioners aim to save resources by leveraging prior tasks to achieve good solutions with only a few evaluations. 
Examples span industrial process tuning and calibration (where competing quality metrics must be balanced), scientific design (e.g., materials discovery and advanced manufacturing with multiple competing properties), and machine-learning system design (e.g., multi-objective hyperparameter and architecture tuning trading off accuracy, latency, and resource use)~\citep{gopakumar2018multi, myung2025multi, pfisterer2022yahpo, eggensperger2021hpobench, Marco2017virtual, herbol2018efficient}.

Despite its promise, meta-learning for multi-objective optimization is technically subtle in the low-data regime. 
First, in multi-objective settings, “what to transfer” is not a single optimum but information about a \emph{Pareto set/front}, and decision-making typically depends on uncertainty-aware criteria (e.g., hypervolume-based utilities)~\citep{Knowles2006parego, daulton2020differentiable, daulton2021parallel}. 
Second, historical data are often scarce and heterogeneous across tasks; transfer must therefore account for \emph{meta-task uncertainty} to avoid overconfident bias from weak or mismatched prior tasks \citep{Dai2022provably, Volpp2020Meta-Learning, feurer2022practical, tighineanu2024scalable}. 
Third, multi-objective problems add another layer: objectives can be correlated, so efficiently using evidence often requires non-trivial \emph{probabilistic multi-output surrogates}. 
Treating objectives independently can waste information precisely when evaluations are precious. 
These challenges leave a key gap: we need meta-learning methods that are (i) uncertainty-aware, (ii) scalable across many meta-tasks, and (iii) able to exploit cross-objective correlations. 
Fully joint multi-task multi-output \ac{GP} models are principled but quickly become infeasible when information from many related tasks is available~\citep{rasmussengp,alvarez2012kernels}, while most scalable alternatives are developed for single-objective transfer and treat objectives independently when adapted to \ac{MOBO}. 

\parahead{Contribution.} 
We introduce \ourmethod, a multi-output \ac{GP} surrogate that meta-learns a target-task prior with learned cross-objective covariance. 
\ourmethod is the first surrogate to simultaneously deliver three properties for meta-\ac{MOBO} in the low-data regime: \emph{fully Bayesian} uncertainty propagation over all meta-task data, \emph{linear scaling} in the number of meta-tasks, and \emph{learned correlations} between objectives. 
The single-objective \priorart~\citep{tighineanu2024scalable} is recovered as a specialization.
\ourmethod's posterior plugs directly into standard \ac{MOBO} pipelines (e.g., hypervolume-based acquisition optimization) \citep{daulton2020differentiable, daulton2021parallel, Knowles2006parego}. 
\ourmethod outperforms, on average, all other baselines in our experiments, demonstrating its practical utility in the regime of scarce and noisy data.

\section{Related work}\label{sec:related}

Most work on meta-learning for \ac{BO} focuses on learning a better surrogate for a single-objective target task. 
A principled approach is to build a \emph{joint} Bayesian model across tasks (e.g., a multi-task \ac{GP}), which yields coherent uncertainty estimates but is computationally prohibitive---scaling cubically in the total number of observations and, at best, quadratically in the task-correlation hyperparameters~\citep{bonilla2007multi, cao2010adaptive, alvarez2012kernels, NIPS2013_swersky, pmlr-v33-yogatama14, EnvGP, NIPS2017_MISO, pmlr-v54-shilton17a, tighineanu2022transfer}. 
To improve scalability, a number of methods rely either on heuristic combinations of per-task surrogates (e.g., \ac{GP} ensembles)~\citep{feurer2022practical, wistuba2018scalable, Dai2022provably} or on building a parametric \ac{GP} prior on the metadata~\citep{poloczek2016warm,ABLR,salinas2020quantile,Wistuba2021few,wang2023hyperbo}. 
These approaches scale better but sacrifice a joint Bayesian treatment and thus principled uncertainty propagation across tasks. 
A recent work, \priorart~\cite{tighineanu2024scalable}, addresses this tension by introducing assumptions that lead to a modular \ac{GP} model: conditioning on meta-data exposes a modular decomposition into $M$ independent meta-task \ac{GP} posteriors and a target-task \ac{GP} prior, enabling scalable and fully Bayesian transfer.

In contrast to the rich single-objective literature, meta-learning methods that directly target multi-objective optimization remain scarce. 
A notable exception is the task-similarity extension of \texttt{MO-TPE} by \citet{watanabe2022speeding}, which transfers knowledge by reweighting the acquisition based on task similarity. 
While scalable, it is based on density ratios and lacks a unified probabilistic multi-output surrogate. 
Recent few-shot surrogate-assisted evolutionary methods for expensive multi-objective optimization meta-learn surrogates~\citep{yufseo}.
However, neither approach learns a correlated multi-output posterior and thus cannot exploit cross-objective dependencies or propagate meta-task uncertainty to the target surrogate in a principled way.
Our work targets this underexplored regime by combining scalable \ac{GP} meta-learning with a multi-output surrogate that models cross-objective dependencies.

\parahead{Desiderata.}
Each of the methods above fails at least one of the requirements (i)--(iii) in \cref{sec:introduction}: per-objective application of \priorart~\citep{tighineanu2024scalable} forgoes (iii); fully joint multi-task \acp{GP} over $(\text{task},\text{objective})$ pairs~\citep{bonilla2007multi} are cubic in the meta-observation count~\citep{alvarez2012kernels}, violating (ii); ensemble combinations of per-task posteriors~\citep{feurer2018scalable, wistuba2018scalable} replace principled conditioning with heuristic weighting, sacrificing (i). \ourmethod, introduced next, is the specific design that satisfies all three.

\section{Method}

\begin{figure}[t]
    \centering
    \includegraphics[width=\linewidth]{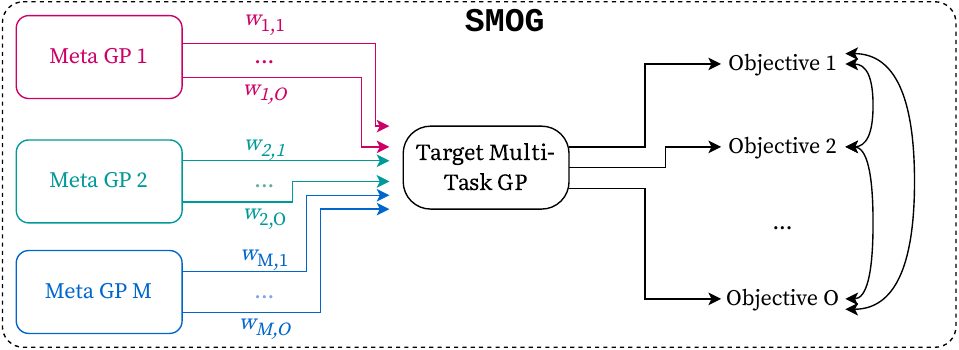}
    \caption{High-level view of \ourmethod: Meta tasks are modeled independently across tasks, with each source task represented by its own multi-output Kronecker-structured \acs{GP}. The target \acs{GP} combines the weighted means and covariance blocks of these source \acsp{GP} with a residual multi-output kernel to form an informative target-task prior.}
    \label{fig:smog-diagram}
\end{figure}

We aim to find a set of Pareto-optimal solutions of a \emph{target} black-box function
$$
\bm{f}_t:\mathcal{X}\to\mathbb{R}^O,
$$
where $\mathcal{X}\subset \mathbb{R}^D$ is the search space of
dimensionality $D$, and $O$ is the number of objectives.
Observations $\bm{y}_n=(y_{n,o})_{o\in\mathcal{O}}$ of $\bm{f}_t$ may be corrupted by independent
zero-mean Gaussian noise,
$\bm{y}_n=\bm{f}_t(\bm{x}_n)+\bm{\epsilon}_n$ with
$\bm{\epsilon}_n\sim\mathcal{N}\!\big(\bm{0},\mathrm{diag}(\sigma_1^2,\ldots,\sigma_O^2)\big)$.
We use the target data $\testdata=\{(\bm{x}_n,\bm{y}_n)\}_{n=1}^{N_t}$ to build a probabilistic model.
We model $\bm{f}_t$ with a multi-output \ac{GP} prior with mean $m_o(\cdot)$ and kernel
$k_{oo'}(\cdot,\cdot)$ for $o,o'\in\mathcal{O}=\{1,\ldots,O\}$.
Conditioned on $\testdata$, the posterior is a \ac{GP} with mean and covariance~\citep{rasmussengp}
\begin{equation}
\begin{split}
\hat{m}_{t,o}(\bm{x})
&= m_o(\bm{x}) +
\bm{k}\!\left((\bm{x},o), \bm{X}_t\right)\left[\bm{K}(\bm{X}_t,\bm{X}_t)+\bm{\Sigma}_\epsilon\right]^{-1}
\left(\bm{Y}_t-\bm{m}(\bm{X}_t)\right), \\
\hat{k}_{t,oo'}(\bm{x},\bm{x}')
&= k\!\left((\bm{x},o),(\bm{x}',o')\right) -
\bm{k}\!\left((\bm{x},o), \bm{X}_t\right) \left[\bm{K}(\bm{X}_t,\bm{X}_t)+\bm{\Sigma}_\epsilon\right]^{-1}
\bm{k}\!\left(\bm{X}_t,(\bm{x}',o')\right),
\end{split}
\label{eq:gp_posterior}
\end{equation}
where $\bm{X}_t=\big((\bm{x}_n,o)\big)_{n=1,\ldots,N_t,\;o\in\mathcal{O}}$ stores points in the search space with their objective index and $\bm{Y}_t=\big(y_{n,o}\big)_{n,o}$ the corresponding observations.
We assume per-objective Gaussian noise
$\bm{\Sigma}_\epsilon=\mathrm{diag}(\sigma_1^2\bm{I}_{N_t},\ldots,\sigma_{O}^2\bm{I}_{N_t})$. To leverage related tasks, we assume access to $M\ge 1$ meta-tasks with datasets
$\mathcal{D}_{1:M}=\bigcup_{m\in\mathcal{M}}\mathcal{D}_m$, where $\mathcal{M}=[M]$.
Each meta-task $m$ provides a dataset $\mathcal{D}_m=\{(\bm{x}_{m,n},\bm{y}_{m,n})\}_{n=1}^{N_m}$ with
$\bm{y}_{m,n}\in\mathbb{R}^O$ and per-objective Gaussian noise
$\bm{y}_{m,n}=\bm{f}_m(\bm{x}_{m,n})+\bm{\epsilon}_{m,n}$,
$\bm{\epsilon}_{m,n}\sim\mathcal{N}\!\big(\bm{0},\mathrm{diag}(\sigma_{m,1}^2,\ldots,\sigma_{m,O}^2)\big)$.
For notational convenience, we collect all observations of meta-task $m$ by concatenating points with their objective index and corresponding outputs as
$\bm{X}_m=\big((\bm{x}_{m,n},o)\big)_{n=1,\ldots,N_m,\;o\in\mathcal{O}}$ and
$\bm{Y}_m=\big(y_{m,n,o}\big)_{n=1,\ldots,N_m,\;o\in\mathcal{O}}$.

\ourmethod extends \priorart~\citep{tighineanu2024scalable} to combine meta-learning with multi-objective optimization, inheriting its task scalability, modularity, and principled uncertainty propagation.
\ourmethod satisfies requirements (i)--(iii) of \cref{sec:introduction} through three design choices: a per-objective additive decomposition with objective-specific weights $w_{mo}$ leading to a closed-form prior after conditioning on metadata in \cref{th:prior_target}, satisfying (i); a structured-coupling assumption that keeps the cross-task parameter count linear in $M$ in \cref{as:meta_target_correlations}, satisfying (ii); and propagating full objective-objective covariance blocks from \cref{as:meta_target_correlations} to \cref{th:prior_target}, satisfying (iii).
In the following, we focus the presentation on the novel components of \ourmethod (the multi-output extension of \cref{as:meta_target_correlations}, objective-specific weights, and the target-task prior with full objective-objective covariance), referring the reader to \citet{tighineanu2024scalable} for a detailed discussion on the shared components (\cref{as:independent_meta,as:hp_independence}).

\subsection{The \texttt{SMOG} kernel}\label{subsec:smog-kernel}
We assume that meta- and target data are described by a multi-task kernel over all tasks and objectives
\begin{equation}
k\left[(\bm{x}, \nu, o), (\bm{x'}, \nu', o')\right] = \sum_{v\in\mathcal{M}^*}\sum_{\theta\in\mathcal{O}}\left[c_{v\theta}\right]_{\nu o,\nu'o'}k_{v\theta}(\bm{x},\bm{x}'),
\label{eq:multitask_kernel}
\end{equation}
where $\mathcal{M}^*=\mathcal{M}\cup{\{t\}}$ is the set of all tasks, $t$ denotes the target index, $\nu,\nu'\in\mathcal{M}^*$ are task indices, $o,o'\in\mathcal{O}$ are indices of the objective, $k_{v\theta}$ are arbitrary kernel functions, and $\bm{C}_{v\theta}$ \ac{PSD} matrices called \emph{coregionalization matrices}, since their entries $\left[c_{v\theta}\right]_{\nu o,\nu'o'}$ model the covariances between two objectives $(\nu, o)$ and $(\nu', o')$ \citep{alvarez2012kernels}.
We impose two assumptions on the multi-task \ac{GP} in \cref{eq:multitask_kernel} that focus learning on the most informative covariance terms.

\begin{assumption}
	Distinct meta-task models are uncorrelated: $\mathrm{Cov}\left(f_{mo}, f_{m'o'}\right) = \delta_{m=m'}k_m\left[(\bm{x},o),(\bm{x'},o')\right]$ for all $m,m' \in \metaset$ and $o,o'\in\mathcal{O}$.
	\label{as:independent_meta}
\end{assumption}

\ourmethod inherits \Cref{as:independent_meta} from \priorart.
It gives scalability in $M$, but by itself does not yield a practical method, since the target-meta couplings can still involve all meta-task kernels and the number of hyperparameters grows quadratically with $M$.
We thus impose a second structural assumption to obtain a practical model in the low-data regime.
\begin{assumption}
	The target-task model is given by a sum of scaled meta-task functions, $\sum_{m\in\mathcal{M}}\tilde{f}_{mo}$, and a residual function, $\tilde{f}_{to}$. 
    Explicitly, $ f_{to} = \tilde{f}_{to} + \sum_{m\in\mathcal{M}}\tilde{f}_{mo} $, with $\left| \mathrm{Corr}(\tilde{f}_{mo}, f_{mo}) \right| = 1$, $\mathrm{Cov}(\tilde{f}_{to}, f_{mo'})=0$, and $\mathrm{Cov}(\tilde{f}_{to},\tilde{f}_{to'})=k_t\left[(\bm{x},o),(\bm{x'},o')\right]$, $m\in\mathcal{M}$, $o,o'\in\mathcal{O}$.
	\label{as:meta_target_correlations}
\end{assumption}

We dissect \cref{as:meta_target_correlations} component by component, and discuss where \ourmethod's novelty lies.
(i)~The additive target decomposition, $f_{to} = \tilde f_{to} + \sum_{m\in\mathcal M}\tilde f_{mo}$, acknowledges that meta-learning may be imperfect and requires a residual target-specific component. Here, \ourmethod extends \priorart's formulation to a per-objective decomposition. Cross-objective dependencies are neglected to enforce explainability: in practice, different objectives may correspond to distinct physical units, and combining them can lead to an unclear representation.
(ii)~The structured-coupling assumption, $\mathrm{Cov}(\tilde f_{to}, f_{mo'}) = 0$, follows \priorart's approach and prevents target-meta couplings from mixing all meta-task kernels, keeping hyperparameter count under control.
(iii)~The link to actual meta-task functions, $|\mathrm{Corr}(\tilde f_{mo}, f_{mo})| = 1$, implies $\tilde f_{mo} = w_{mo} f_{mo}$ and yields the practical form $f_{to} = \tilde f_{to} + \sum_{m} w_{mo} f_{mo}$ with $w_{mo} \in \mathbb{R}$. \ourmethod's generalization discards cross-objective contributions for the same reasons as in (i). This meta-function link is a specific choice of the \ourmethod/\priorart family; other links are possible, but they typically lead to a less interpretable target-task prior in \cref{th:prior_target}. We leave the exploration of such alternatives to future work.

\begin{restatable}{lemma}{coregmatrices}
Applying \cref{as:independent_meta,as:meta_target_correlations} to \cref{eq:multitask_kernel} yields a sparse structure with the following non-zero entries of the coregionalization matrices
\begin{equation}
\begin{aligned}
[c_{t\theta}]_{to,to'} &= [h_{t\theta}]_{oo'}, &
[c_{m\theta}]_{mo,mo'} &= [h_{m\theta}]_{oo'}, \\
[c_{m\theta}]_{mo,to'} &= w_{mo'}[h_{m\theta}]_{oo'}, &
[c_{m\theta}]_{to,to'} &= w_{mo}w_{mo'}[h_{m\theta}]_{oo'} ,
\end{aligned}
\label{eq:coreg_matrices_after_assumptions}
\end{equation}
where $[c_{\cdot\theta}]_{a,b}=[c_{\cdot\theta}]_{b,a}$ and $\left[h_{v\theta}\right]_{oo'}$ are \ac{PSD} in the basis $(o,o')$.
\label{lemma:coreg_matrices}
\end{restatable}
See \cref{app:coreg_matrices} for a proof and an illustrative example. 
This sparse structure decreases the number of kernel parameters from $\mathcal{O}\left[M^3O^3\right]$ in \cref{eq:multitask_kernel} to $\mathcal{O}\left[MO^3\right]$ in \cref{eq:coreg_matrices_after_assumptions}, which is linear in $M$. In other words, the coregionalization matrices are now populated by $O\times O$ blocks of $\left[h_{v\theta}\right]_{oo'}$.

\begin{restatable}{lemma}{jointkernel}
\cref{as:independent_meta,as:meta_target_correlations} with $w_{mo}\in\mathbb{R}$ for $m\in\mathbb{R}$ and $o\in\mathcal{O}$ yield a valid kernel given by
\begin{equation}
\begin{split}
k_\ourmethod&\left[(\bm{x},\nu,o), (\bm{x'},\nu',o')\right] = \sum_{v\in\mathcal{M}^*}g_{v}(\nu,o)g_{v}(\nu',o')k_v\left[(\bm{x},o),(\bm{x'},o')\right],
\end{split}
\label{eq:kernel_joint_final_index}
\end{equation}
where $g_v(\nu,o)$ is one if $v=\nu$, $w_{mo}$ if $(v = m) \land (\nu=t)$, and zero otherwise.
\label{lemma:joint_smog_kernel}
\end{restatable}
See \cref{app:smog_joint_kernel} for a proof. By \cref{lemma:joint_smog_kernel}, we have a valid joint kernel defining the prior distribution over all meta- and target functions \emph{before} observing any data. 
The number of parameters in our kernel scales linearly with $M$. 
The modular and scalable nature of \ourmethod is revealed when conditioning \cref{eq:kernel_joint_final_index} on the metadata, yielding a valid \ac{GP}, as we show next.
\begin{restatable}{theorem}{targetprior}
	Under a zero-mean \ac{GP} prior with the multi-task kernel given by \cref{eq:kernel_joint_final_index}, the distribution of the target-task objectives conditioned on the metadata is 
	\begin{equation}
	f_{to} \mid \mathcal{D}_{1:M}, \bm{x} \sim \mathcal{GP}(m_{t,\ourmethod}(\bm{x},o), k_{t,\ourmethod}[(\bm{x}, o), (\bm{x'},o')])
	\end{equation}
	with
    \begin{equation}
    \begin{split}
    m_{t,\ourmethod}(\bm{x}, o) &= \sum_{m\in\mathcal{M}}w_{mo}\hat{m}_{mo}(\bm{x}), \\
	\kk[t,\ourmethod]{(\bm{x}, o), (\bm{x'}, o')} &= \kk[t]{(\bm{x},o),(\bm{x'},o')}+ \sum_{m\in\mathcal{M}}w_{mo}w_{mo'}\hat{k}_{moo'}(\bm{x}, \bm{x'}),
    \label{eq:target_prior}
    \end{split}
    \end{equation}
	where $\hat{m}_{mo}(\bm{x})$ and $\hat{k}_{moo'}(\bm{x},\bm{x'})$ are the posterior mean and covariance functions of the individual multi-output meta-task \acp{GP} conditioned only on their corresponding data, $\mathcal{D}_m$.
	\label{th:prior_target}
\end{restatable}
See \cref{app:target_prior} for a proof. 
According to \cref{th:prior_target}, the prior distribution of the target-task is also a \ac{GP}, given by the weighted sum of the meta-task posteriors. 
The prior mean function is analogous to \priorart~\citep{tighineanu2024scalable}, with a per-objective weighting of meta-task posterior means. \ourmethod's covariance prior is a generalization of \priorart's: while the latter weighs per-task posterior covariances with a scalar $w_m^2$, \ourmethod uses products $w_{mo}w_{mo'}$ that couple the full objective--objective covariance blocks. 
\ourmethod's residual kernel, $k_t$, learns correlations that cannot be captured by the meta-learned contribution alone. 
According to \cref{eq:target_prior}, meta-tasks that align with the target receive large weights, while unrelated ones are downweighted. 
\ourmethod can therefore quickly learn in the presence of even a few similar meta-tasks. 
The prior in \cref{eq:target_prior} is conditioned on $\testdata$ via \cref{eq:gp_posterior} to obtain the target-task posterior.
If the objectives are independent, \ourmethod reduces to describing each objective with a \priorart model.
\begin{figure}[t]
    \centering
    \includegraphics[width=\linewidth]{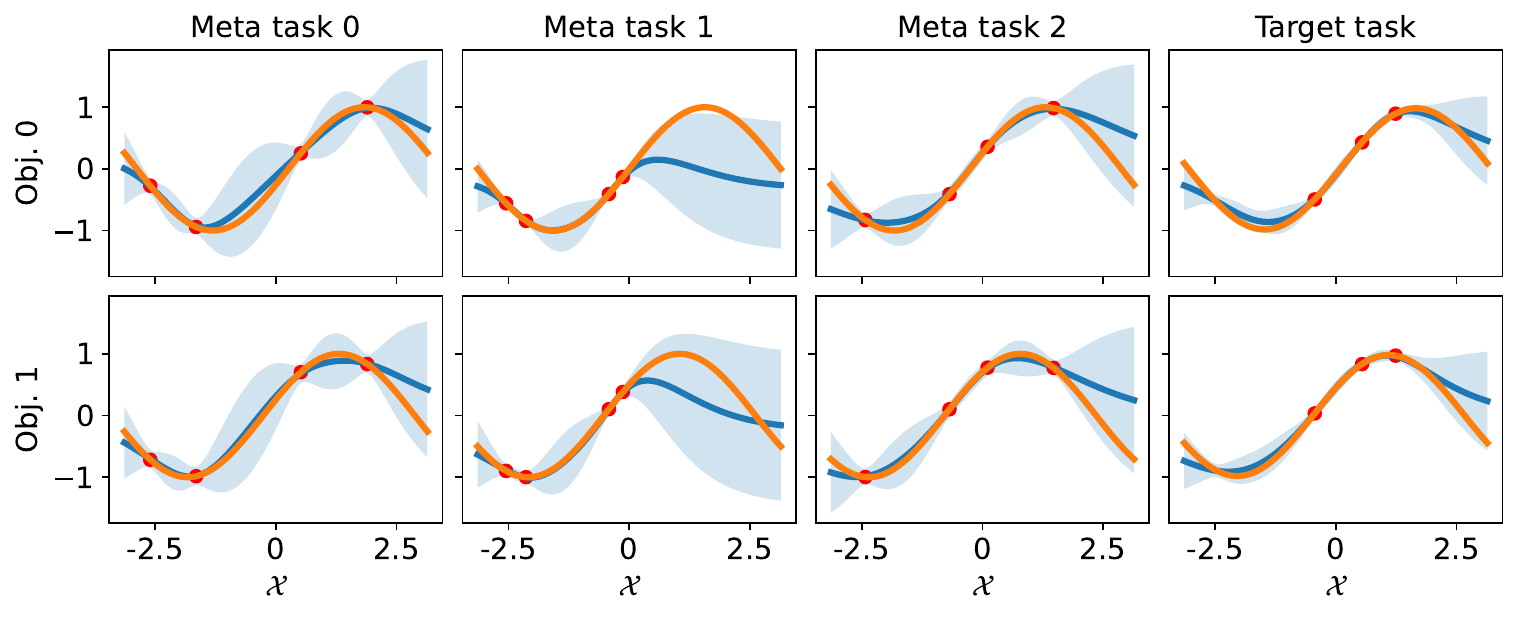}
    \caption{Example of a Sinusoidal function with two outputs (rows), three source tasks (columns 1--3), and one target task (column 4). \ourmethod learns a strong target-task posterior (solid blue line and shaded areas) of the objective function (solid orange line) by leveraging meta data (red dots for meta tasks 0--2) to learn an informative target-task prior, which is further refined by conditioning on the target data (red dots in the rightmost column). See \cref{app:benchmark-details} for details on the benchmark.}
    \label{fig:sinusoidal_results}
\end{figure}
\Cref{fig:sinusoidal_results} visualizes how \ourmethod meta-learns on multi-objective problems.

\parahead{Complexity analysis.} 
A key feature of \ourmethod is that its computational complexity scales linearly in $M$. 
Consider the cost incurred by the different steps of \cref{alg:overview}. 
Line~\ref{alg:overview:meta-gps} involves the inversion of meta-task data with an overhead $\mathcal{O}(O^3\sum_{\meta\in\metaset}N_m^3)$. 
This happens only once during pre-training, the results are cached.
To construct the prior in line~\ref{alg:overview:test-prior}, we evaluate the meta-task posterior at $\bm{X}_\test$ with an overhead $\mathcal{O}(O^3N_\test^2\sum_{m\in\mathcal{M}}N_\meta + O^3N_\test\sum_{m\in\mathcal{M}}N_\meta^2)$. 
Finally, line~\ref{alg:overview:test-likelihood} involves inverting the target-task kernel matrix, taking $\mathcal{O}(O^3N_\test^3)$. 
All these terms scale linearly in $M$ and are cheap to evaluate in the practically relevant regime of small $N_t$ and $O$, and moderate data per meta-task.
We empirically validate this analysis in \cref{app:runtime}.

\subsection{Hyperparameter inference} 
As for \priorart, we assume that the \acp{HP} of each meta-task model, $\bm{\theta}_m$, are independent of the target data~\citep{bayarri2009modularization}.
\begin{assumption}
	For all meta-tasks $m\in\mathcal{M}$, $p(\bm{\theta}_m \mid \mathcal{D}_m,\testdata) = p(\bm{\theta}_m \mid \mathcal{D}_m)$.
	\label{as:hp_independence}
\end{assumption}
\cref{as:hp_independence} expresses a no-feedback modularization: meta-task \acp{HP} are inferred without feedback from the target data. 
In the regime we consider, where target data are scarce relative to the metadata, this reduces undesirable target-to-meta feedback and leads to a modular posterior
\begin{equation}
    p(\bm{\theta}_{1:\nummeta}, \bm{\theta}_\test \mid \mathcal{D}) = \left[\prod_{m=1}^Mp(\bm{\theta}_m \mid \mathcal{D}_m)\right]p(\bm{\theta}_\test \mid \mathcal{D}),
    \label{eq:modularized_posterior}
\end{equation}
with $\mathcal{D} = \testdata \cup \metadata$. 
Being a \ac{GP}, \ourmethod is amenable to fully Bayesian inference. 
In our experiments, however, we follow a modular plug-in approximation for the maximum-likelihood / MAP~\citep{murphy2002estimation}:
\begin{equation}
    \bm{\theta}_m^\star = \argmax_{\theta_m}\log p(\bm{Y}_m \mid \bm{X}_m, \bm{\theta}_m), \quad \forall m\in\mathcal{M},
    \label{eq:meta_likelihood_optimization}
\end{equation}
followed by
\begin{equation}
    \bm{\theta}_\test^\star = \argmax_{\bm{\theta}_\test}\log p(\bm{Y}_\test \mid \mathcal{D}_{1:\nummeta}, \bm{X}_\test, \bm{\theta}_\test, \bm{\theta}_{1:\nummeta}^\star).
    \label{eq:target_likelihood_optimization}
\end{equation}
This two-stage optimization does not recover the maximizer of the joint \ac{MLL} over $(\bm{\theta}_{1:\nummeta}, \bm{\theta}_\test)$; rather, it is the optimization counterpart of the modularized posterior in \cref{eq:modularized_posterior}. 
We summarize \ourmethod in \cref{alg:overview} and show a conceptual diagram in \cref{fig:smog-diagram}.

\begin{algorithm}[tb]
	\caption{\ourmethod}
	\label{alg:overview}
	\begin{algorithmic}[1]
		\State {\bfseries Input:} metadata $\metadata = \cup_{m \in \metaset} \mathcal{D}_m$ \label{alg:overview:input}
        \State {\bfseries Output:} meta-learned multi-objective model $\mathcal{GP}(\testdata,\metadata)$
		\State Train individual \acp{GP} per meta-task and optimize $\bm{\theta}_m$ \label{alg:overview:meta-gps}
		\State Construct the target-task prior as in \cref{eq:target_prior}, and cache $\bm{\hat{m}}_{mo}(\bm{X}_\test)$ and $\bm{\hat{K}}_{moo'}(\bm{X}_\test, \bm{X}_\test)$ \label{alg:overview:test-prior}
		\State Optimize the target-task \acp{HP} $\bm{\theta}_\test$ as in \cref{eq:target_likelihood_optimization} \label{alg:overview:test-likelihood}
		\State Condition the prior on $\testdata$ to obtain the posterior distribution for $\bm{f}_\test$ as in \cref{eq:gp_posterior} \label{alg:overview:test-posterior}
	\end{algorithmic}
\end{algorithm}

\subsection{Equicorrelated kernel}\label{subsec:equicorrelated}
\ourmethod has a modular structure: conditioning on the metadata collapses the original monolithic \ac{GP} into $M+1$ interacting \acp{GP} in \cref{th:prior_target}. The user has full flexibility to design these individual modules depending on prior knowledge or requirements. In our experiments, we use a conventional Kronecker structure for the meta-task \acp{GP}~\cite{bonilla2007multi}. To model the correlations of residuals of the target task, we adopt a separable Kronecker form $\mathrm{Cov}(\tilde{f}_{to},\tilde{f}_{to'}) = k_{\textrm{Obs}}(\bm{x}, \bm{x}')\, K_\textrm{Obj}[o, o']$ in which $k_{\textrm{Obs}}$ is a within-task input kernel and $K_\textrm{Obj} \in \mathbb{R}^{O \times O}$ is an objective-objective coupling matrix.
A general $K_\textrm{Obj}$ has $\mathcal{O}(O^2)$ free entries. 
For the target-task budgets, we aim for $\mathcal{O}(\text{tens})$ observations, which impedes the identification of free entries.
We therefore restrict $K_\textrm{Obj}$ to an \emph{equicorrelated} form
\begin{equation}
K_\textrm{Obj}[i, j] =
\begin{cases}
\sigma_i^2, & i = j, \\
\rho\, \sigma_i\, \sigma_j, & i \neq j,
\end{cases}
\qquad \rho \in \left[-\tfrac{1}{O-1},\, 1\right],
\label{eq:equicorrelated}
\end{equation}
in which a single correlation parameter $\rho$ is shared across all pairs of distinct objective residuals of the target task, and each objective retains its own marginal scale $\sigma_i$.
This admits the rank-one-plus-diagonal decomposition $K_\textrm{Obj} = (1{-}\rho)\,\mathrm{diag}(\bm{\sigma}^2) + \rho\, \bm{\sigma}\bm{\sigma}^\top$, is \ac{PSD} on the stated $\rho$ interval, and reduces the free parameters of $K_\textrm{Obj}$ from $\mathcal{O}(O^2)$ to $O + 1$.
Compared to the generic low-rank form $K_\textrm{Obj} = BB^\top + \mathrm{diag}(\bm{v})$ with $B \in \mathbb{R}^{O \times O}$, equicorrelation imposes a single shared cross-objective correlation of the target task residuals while remaining identifiable on the data budgets that motivate meta-learning in the first place.
We treat $\rho$ as a hyperparameter with a Beta prior (\cref{app:experimental-setup}) and fit it jointly with the per-objective scales $\sigma_i$ during target-task \ac{GP} training.

\section{Experimental evaluation}
\label{sec:experiments}

We study the performance of \ourmethod relative to a wide range of optimization algorithms across benchmarks that reflect synthetic and real-world scenarios.
We initialize each optimizer with a single uniformly randomly selected configuration, reflecting the data scarcity in meta-learning scenarios, and run each optimizer $50$ times with different random seeds unless stated otherwise.
Every iteration following the initial sample uses \texttt{LogExpectedHypervolumeImprovement}~\citep{ament2023unexpected} as the \acl{AF}; we study the sensitivity to this choice in \cref{app:acqfunc-comparison}.
\texttt{FCNetTabularBenchmark} has only two objectives and was therefore excluded from all 4-objective experiments.
Further details on the compute environment, reference point, hyperpriors, and marginal-likelihood optimization are provided in \cref{app:experimental-setup}; mixed-space acquisition function optimization is described in \cref{app:mixed-space-af}.

For each group of runs sharing the same benchmark, number of objectives, and target task, we normalize the observed objective values to $[0,1]$ per objective using group-wide minimum and maximum values computed from all pooled observations, and set the \ac{HV} reference point to the origin in this normalized space (corresponding to the worst observed value per objective across all runs).
The \emph{\ac{HV} gap} is $1 - \widehat{\mathrm{HV}}_t$, where $\widehat{\mathrm{HV}}_t \in [0,1]$ is the normalized \ac{HV} of the Pareto front built from all solutions found up to \ac{BO} iteration~$t$; lower values indicate a better Pareto front.
The \emph{cumulative \ac{HV} regret} $\sum_{s=1}^{t}(1 - \widehat{\mathrm{HV}}_s)$ is the cumulative sum of per-iteration \ac{HV} gaps and penalizes algorithms that discover good fronts late, rewarding fast convergence.

\subsection{Benchmarks}\label{subsec:benchmarks}
We evaluate the performance of \ourmethod in controlled synthetic and real-world settings.
We use 8 meta-tasks and 16 observations per meta-task in the main text, unless otherwise stated.

\parahead{Adapted \hartmann benchmark.}
We define an adapted variant of the \hartmann benchmark to allow for meaningful multi-task, multi-objective optimization.
To obtain a multi-task multi-objective setting, we (i) sample a separate coefficient vector $\bm{\alpha}_m$ per task, yielding related but distinct tasks, and (ii) apply small objective-specific input shifts $\bm{\varepsilon}_o$ so that different objectives attain their optima at different locations.
We additionally study the behavior of \ourmethod for different numbers of meta tasks and observations per meta task in Appendix~\ref{app:ablations}.
A complete specification is provided in Appendix~\ref{app:benchmark-details}.

\parahead{Adapted \branincurrin benchmark.}
We consider an adapted variant of the 2D \branincurrin multi-objective benchmark~\citep{irshad2024leveraging}.
To obtain a multi-task multi-objective setting, we (i) perturb the underlying function parameters independently per task and objective, yielding related but distinct tasks, and (ii) apply small objective-specific input shifts so that objectives attain their optima at different locations.
We report objectives in a maximization form (negated values) to match standard hypervolume-based evaluation.
A complete specification is provided in Appendix~\ref{app:benchmark-details}.

\parahead{Tabular HPO benchmarks.}
To evaluate \ ourmethod's performance in real-world settings, we investigate it on HPOBench benchmarks~\cite{klein2019tabular}.
The goal in HPOBench is to jointly optimize a neural network architecture and its hyperparameters.
The performance of a configuration can be observed on four different datasets: \textit{Slice Localization}, \textit{Protein Structure}, \textit{Naval Propulsion}, and \textit{Parkinson's Telemonitoring}.
The benchmark has two objectives: validation MSE and runtime.
We always run for the highest-fidelity setting, corresponding to the maximum number of epochs (100). 
We use one dataset as the target task, the other three as meta tasks, and 16 observations per meta task.

\parahead{Terrain benchmark.}
In this benchmark, we study \ac{UAV} trajectory optimization problems~\cite{shehadeh2025benchmarking}.
The goal is to find a trajectory of 20 three-dimensional waypoints through one of 56 predefined landscapes that minimizes four target metrics: path length cost, obstacle avoidance cost, altitude cost, and smoothness cost.
Other landscapes are used as meta tasks.
We use the first three landscapes as target tasks. 
For each target task, we uniformly sample 8 meta tasks from the remaining landscapes and use 64 observations per meta task.
See Appendix~\ref{app:benchmark-details} for details.

\begin{figure*}
    \centering
    \includegraphics[width=\linewidth]{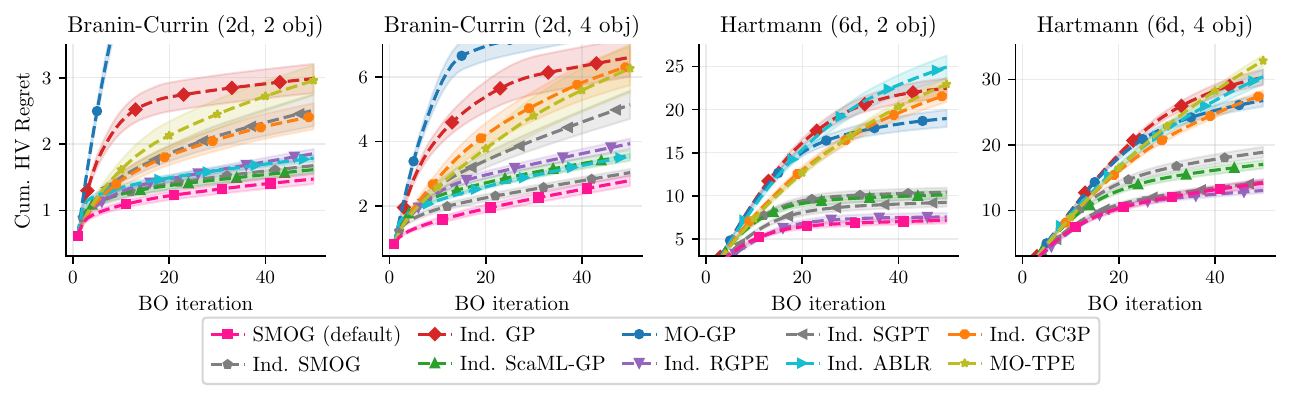}
    \caption{%
        Cumulative \ac{HV} regret (lower is better) on the synthetic benchmarks.
        Columns correspond to benchmark and objective-count combinations:
        BraninCurrin~(2\,obj), BraninCurrin~(4\,obj), \hartmann~(2\,obj), \hartmann~(4\,obj).
        All results use 8 meta-tasks and 16 observations per meta-task.
        The corresponding \ac{HV} gap is shown in \cref{fig:synthetic_hv_appendix} in the Appendix.%
    }
    \label{fig:synthetic_results}
\end{figure*}
\begin{figure}
    \centering
    \includegraphics[width=\linewidth]{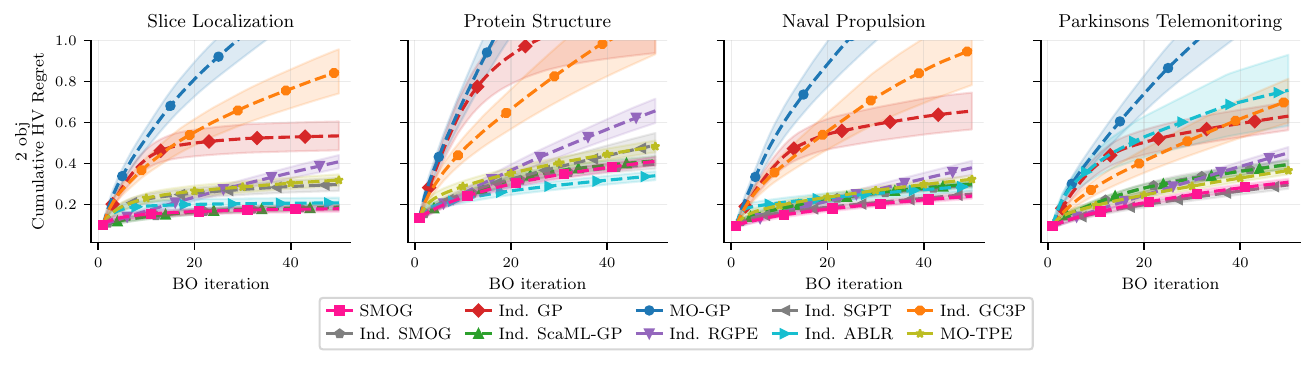}
    \caption{Cumulative \ac{HV} regret (lower is better) on the HPOBench benchmarks.
        Columns correspond to the four target datasets.
        All results use 3 meta-tasks and 16 observations per meta-task.}
    \label{fig:fcnet_results}
\end{figure}

\subsection{Models}\label{subsec:models}

We compare \ourmethod against the following baselines: \MOGP, a multi-task \ac{GP} with Kronecker structure~\cite{bonilla2007multi}; \IndGP, a \ac{GP} model without meta-learning and independent outputs; \IndSCAML, a \priorart model with independent outputs~\cite{tighineanu2024scalable}; \MOTPE, a \ac{TPE}-based model that naturally handles multiple objectives and meta-learning~\cite{watanabe2022speeding}; \IndRGPE and \IndSGPT, meta-learning models using linear combinations of \acp{GP}~\citep{feurer2018scalable,wistuba2018scalable}; \IndABLR, adaptive Bayesian linear regression~\citep{ABLR}; \IndGCThreeP, which uses Gaussian copulas to map observations from different tasks to comparable distributions~\citep{salinas2020quantile}; and \IndSMOG, an ablation variant of \ourmethod that replaces the Kronecker source \acp{GP} with independent per-objective source \acp{GP} while retaining the multi-objective target head. 
Models starting with ``\texttt{Ind.}'' model distinct objectives independently.
We additionally evaluate each model's surrogate quality (\ac{NLPD} and \ac{RMSE}) in isolation, prior to any \ac{BO} iterations, in \cref{app:surrogate-quality}.
Full implementation details are provided in \cref{app:model-details}.

\subsection{Results}

\parahead{Synthetic benchmarks.}
\Cref{fig:synthetic_results} reports cumulative \ac{HV} regret on \branincurrin and \hartmann, each with 2 and 4 objectives; the corresponding \ac{HV} gap is in \cref{fig:synthetic_hv_appendix}.
The non-meta-learning baselines \IndGP and \MOGP are clearly separated from the rest, accumulating the highest regret across all configurations.
Among the metadata-aware methods, \IndSCAML, \IndRGPE, \IndSGPT, and \IndGCThreeP achieve a substantial initial speedup; \IndABLR is competitive on \branincurrin but loses ground on \hartmann. \ourmethod attains the lowest or near-lowest cumulative regret on all four configurations.
The ablation \IndSMOG---which keeps \ourmethod's meta-learning structure but treats objectives independently---closely tracks \ourmethod on most benchmarks, confirming meta-learning as the dominant source of gains; the gap to \IndSMOG widens on \hartmann (4 obj), where cross-objective coupling provides additional benefit.
\IndRGPE matches \ourmethod on \hartmann but is mediocre on \branincurrin.
\texttt{MO-TPE}, the only other method combining meta-learning with multi-objective optimization, beats the non-meta-learning baselines early but is otherwise not competitive.
\Cref{app:exemplary_hartmann_front} presents an analysis of the \hartmann solutions in more detail.

\parahead{HPOBench benchmarks.}
\Cref{fig:fcnet_results} shows the performance of \ourmethod and competitors on the HPOBench benchmark.
For space reasons, we defer the \ac{HV} gap results on the individual datasets to \cref{fig:hpobench_appendix_results} in the Appendix.
As observed on the synthetic benchmarks, \MOGP and \IndGP are not competitive and accumulate substantially more regret than the metadata-leveraging methods on all four datasets.
All other methods leverage metadata and outperform \IndGP and \MOGP by a wide margin.
\ourmethod performs well throughout the benchmark suite. 
On the slice localization and protein structure problems, \ourmethod achieves performance comparable to \IndSMOG and \IndSCAML, which perform worse on the other two problems.
On the protein structure and parkinsons telemonitoring problems, \ourmethod is outperformed by \IndABLR and \IndSGPT, respectively.
However, these methods struggle with other problems (see \cref{fig:synthetic_results}).
The Pareto fronts of the solutions found across all methods and datasets are shown in \cref{fig:pareto-hpobench} in the Appendix.

\parahead{Terrain benchmark.}
Next, we study \ourmethod's performance on the terrain benchmark.
\Cref{fig:terrain_cumregret} shows the cumulative \ac{HV} regret for all three target tasks and both objective configurations (2- and 4-objective).
\Cref{fig:terrain_hv} in the Appendix shows the corresponding \ac{HV} gap.
\begin{figure*}[t]
    \centering
    \includegraphics[width=.95\linewidth]{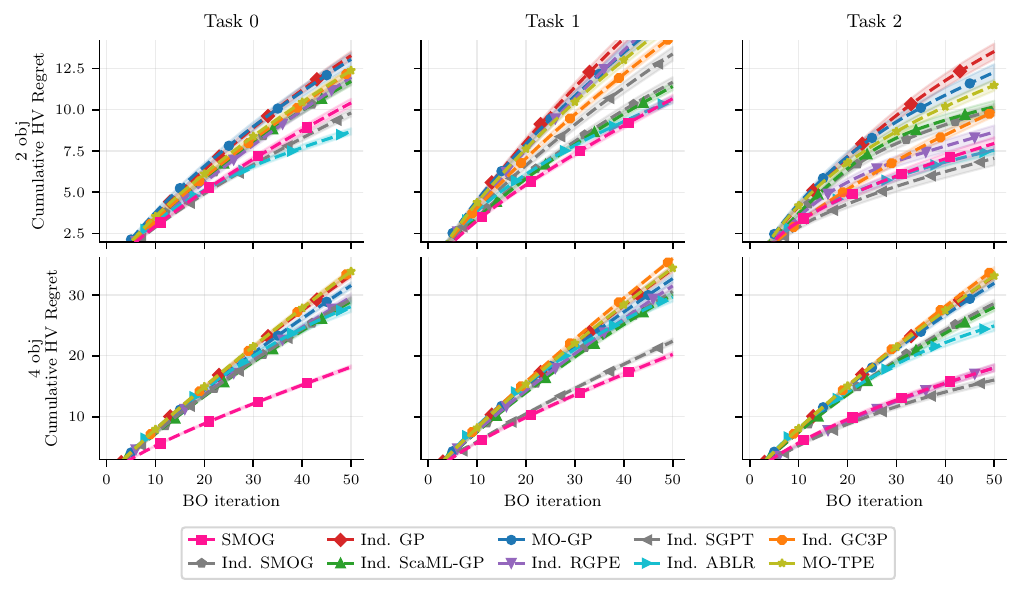}
    \caption{%
        Cumulative \ac{HV} regret (lower is better) on the Terrain benchmark.
        Rows correspond to 2-objective (top) and 4-objective (bottom) variants;
        columns correspond to target tasks 0, 1, and 2.
        All results use 8 meta-tasks and 64 observations per meta-task.%
    }
    \label{fig:terrain_cumregret}
\end{figure*}
The terrain benchmark paints a similar picture to HPOBench: \IndABLR or \IndSGPT occasionally outperform \ourmethod but struggle on other problems, while \ourmethod performs well overall.
Particularly striking is the wide gap between \ourmethod and other methods on the 4-objective variant of task 0.
\ourmethod's advantage increases when transitioning from 2 to 4 objectives, supporting its effectiveness in modeling objective correlations.

\section{Conclusion}

Many impactful applications require optimizing competing objectives:
in aerospace engineering, one seeks a lightweight structure with high strength, while in machine learning, one aims for accurate yet small models.
In many cases, practitioners can use data from related tasks or experiments to quickly find a good solution for a new task.
In this paper, we introduce \ourmethod---a scalable meta-learning algorithm for multi-objective black-box optimization problems.
\ourmethod leverages observations from related tasks and models cross-objective correlations to construct an informative target-task prior, improving sample efficiency in the initial BO iterations.
\ourmethod is principled, with a clear theoretical motivation, and performs robustly when studied in practice.
As such, \ourmethod fills a gap by providing a practical and principled algorithm that combines meta-learning and multi-objective optimization.


\parahead{Limitations.}
\ourmethod's limitations are dictated by its assumptions. \cref{as:independent_meta} achieves scalability but ignores correlations across meta-tasks; \cref{as:meta_target_correlations} restricts the number of \acp{HP} but also limits the number of meta-learning channels; and \cref{as:hp_independence} achieves a practical implementation but blocks the feedback from the target data to meta-task \acp{HP}.
While well-motivated in the regime of scarce and noisy data, more flexible neural-network-based models may be able to meta-learn more complex relationships when ample data is available.
In addition, \ourmethod scales to a large number of tasks, but each task model remains an exact \ac{GP} and scales cubically with the number of data points, so large datasets would require sparse or approximate \ac{GP} extensions.
Finally, \ourmethod assumes that all tasks share the same search and objective spaces.

\parahead{Broader impact.}
This paper presents foundational research on sample-efficient optimization.
On the positive side, more data-efficient optimization reduces the number of expensive experiments needed in engineering and science, thereby saving time, energy, and resources.
On the negative side, we are not aware of direct pathways from this work to harmful applications.

\begin{ack}
Calculations (or parts of them) for this publication were performed on the HPC cluster PALMA II of the University of Münster, subsidized by the DFG (INST 211/667-1).
The authors gratefully acknowledge the computing time granted by the Resource Allocation Board and provided on the supercomputer Emmy/Grete at NHR-Nord@Göttingen as part of the NHR infrastructure. The calculations for this research were conducted with computing resources under the project  nhr\_nw\_test.
The authors gratefully acknowledge the computing time provided to them at the NHR Center NHR4CES at RWTH Aachen University (project number p0026398). This is funded by the Federal Ministry of Education and Research, and the state governments participating on the basis of the resolutions of the GWK for national high performance computing at universities (\url{www.nhr-verein.de/unsere-partner}).
\end{ack}

\FloatBarrier

\bibliography{references}
\bibliographystyle{unsrtnat}

\newpage
\appendix
\acresetall

\section{Kernel properties}
\label{app:kernel_properties}
Here we prove the key properties of \ourmethod's kernel encoded in \cref{lemma:coreg_matrices,lemma:joint_smog_kernel}.

\subsection{Coregionalization matrices}
\label{app:coreg_matrices}
Here we prove \cref{lemma:coreg_matrices}, which we restate below.
\coregmatrices*
\begin{proof}

We begin by showing that the coregionalization matrices in \cref{eq:coreg_matrices_after_assumptions} are uniquely defined by \cref{as:independent_meta,as:meta_target_correlations}. We collect all terms in \cref{eq:multitask_kernel}:
\begin{align}
    \kk{ (\x,\meta,o), (\x',\meta,o') } &= \cov{ f_{\meta o}(\x), f_{\meta o'}(\x') } = \kk[\meta]{(\x,o), (\x',o')} \quad \text{(\cref{as:independent_meta})} \\
    \kk{ (\x,\meta,o), (\x',\meta'\neq\meta,o') } &= 0 \quad \text{(\cref{as:independent_meta})} \\
    \kk{ (\x,t,o), (\x',\meta,o') } &= \cov{f_{to}(\x), f_{mo'}(\x')} \\
        &= \cov{\tilde{f}_{to} + \sum_{\meta'\in\metaset}\tilde{f}_{\meta'o}, f_{\meta o'}} \quad \text{(\cref{as:meta_target_correlations})} \nonumber \\
        &= \sum_{\meta'\in\metaset}w_{\meta'o}\cov{f_{\meta'o} , f_{\meta o'}} \quad \text{(\cref{as:meta_target_correlations})} \nonumber \\
        &= w_{\meta o}\kk[\meta]{ (\x,o), (\x',o') } \quad \text{(\cref{as:independent_meta})},
\end{align}
where we have made use of the fact that the covariance is a bilinear function, and that the perfect correlation $\corr{f_{\meta o}(\x), \tilde{f}_{\meta o}(\x')} = \pm 1$ in \cref{as:meta_target_correlations} implies that $\tilde{f}_{\meta o}=w_{\meta o}f_{\meta o} + c$ with $w_{mo}, c \in \mathbb{R}$. Finally, we have
\begin{equation}
\begin{split}
    &\kk{ (\x,\test,o), (\x',\test,o') } = \cov{f_{\test o}(\x), f_{\test o'}(\x')} \\
    &= \cov{\tilde{f}_{\test o}(\x), \tilde{f}_{\test o'}(\x')} + \sum_{\meta,\meta'\in\metaset}w_{\meta o}w_{\meta'o'} \cov{f_{\meta o}(\x), f_{\meta'o'}(\x')}  \quad \text{(\cref{as:meta_target_correlations})} \\
    &= \kk[\test]{(\x,o), (\x',o')} + \sum_{\meta\in\metaset}w_{\meta o}w_{\meta o'}\kk[m]{(\x,o),(\x',o')} \quad \text{(\cref{as:independent_meta,as:meta_target_correlations})}
\end{split}
\end{equation}
By collecting the coefficients corresponding to $k_m$ and $k_t$, we obtain the coregionalization matrices of \ourmethod in \cref{eq:coreg_matrices_after_assumptions,eq:coreg_matrices_example}. 

These coregionalization matrices are elements of the original \ac{MTGP}, implying that they are \ac{PSD} by definition. To see this, consider the expression for the meta-task block of the \ac{MTGP} kernel in \cref{eq:multitask_kernel} (excluding the target task and all meta-task-to-target-task couplings):
\begin{equation}
k\left[(\bm{x},m,o),(\bm{x}',m',o')\right]= \delta_{m=m'}\sum_{\theta\in\mathcal{O}}\left[h_{m\theta}\right]_{oo'}k_{m\theta}(\bm{x},\bm{x'})
\equiv \delta_{m=m'}k_m\left[(\bm{x},o),(\bm{x'},o')\right],
\end{equation}
where $\left[h_{m\theta}\right]_{oo'} = \left[c_{m\theta}\right]_{mo,mo'}$. This formulation does not make any additional assumptions about the kernel across objectives, $k_m\left[(\bm{x},o),(\bm{x'},o')\right] = \sum_{\theta\in\mathcal{O}}\left[h_{m\theta}\right]_{oo'}k_{m\theta}(\bm{x},\bm{x'})$. The elements $\left[h_{m\theta}\right]_{oo'}$ appear directly in \cref{eq:coreg_matrices_after_assumptions} and are by definition \ac{PSD}. Similar reasoning can be applied to the other elements, implying that all entries of \ourmethod's coregionalization matrices in \cref{eq:coreg_matrices_after_assumptions} are \ac{PSD}.
\end{proof}

\parahead{Example for two meta-tasks and objectives.} We have the following coregionalization matrices in the basis $(m=1,o=1), (m=1,o=2), (m=2,o=1), (m=2,o=2), (t,o=1), (t,o=2)$:
\begin{equation}
\begin{split}
\bm{C}_{1\theta} &= \begin{pmatrix}
\bm{H}_{1\theta} & \zeromatrix{2} & W_1 \odot \bm{H}_{1\theta} \\
\zeromatrix{2}                       & \zeromatrix{2} & \zeromatrix{2} \\
W_1^\top \odot \bm{H}_{1\theta}  & \zeromatrix{2} & W_1^\times \odot \bm{H}_{1\theta}
\end{pmatrix},\qquad \bm{C}_{2\theta} = \begin{pmatrix}
\zeromatrix{2} & \zeromatrix{2}   & \zeromatrix{2} \\
\zeromatrix{2} & \bm{H}_{2\theta} & W_2 \odot \bm{H}_{2\theta} \\
\zeromatrix{2} & W_2^\top \odot \bm{H}_{2\theta} & W_2^\times \odot \bm{H}_{2\theta}
\end{pmatrix},\\
\bm{C}_{t\theta} &= \begin{pmatrix}
\zeromatrix{2}    & \zeromatrix{2}     & \zeromatrix{2} \\
\zeromatrix{2}    & \zeromatrix{2}     & \zeromatrix{2} \\
\zeromatrix{2}    & \zeromatrix{2}     & \bm{H}_{t\theta}
\end{pmatrix},\qquad W_i=\wmatrix{i},\qquad W_i^\times=\wmatrixSQ{i},
\end{split}
\label{eq:coreg_matrices_example}
\end{equation}
where $\odot$ is the elementwise (Hadamard) product. Here, $\bm{H}_{v\theta}$ is matrix notation for $\left[h_{v\theta}\right]_{oo'}$ and describes cross-objective correlations.

\subsection{\ourmethod kernel}
\label{app:smog_joint_kernel}
Next, we prove \cref{lemma:joint_smog_kernel}, which we restate below.
\jointkernel*
\begin{proof}
We start off by using the results of \cref{lemma:coreg_matrices}, \cref{eq:coreg_matrices_example}. The first notational simplification  for the coregionalization matrices can be done by pulling the common $\bm{H}_{v\theta}$ term outside the big matrix. To do this, we introduce a custom matrix product $\boxtimes$ between matrices $\bm{A}$ of size $(\alpha+\beta)\times(\alpha+\beta)$ and $\bm{B}$ of size $\beta \times \beta$. Each $\beta \times \beta$ block of matrix $\bm{A}$ is Hadamard-multiplied with matrix $\bm{B}$:
\begin{equation}
(\bm{A}\boxtimes\bm{B})_{\beta\times\beta} = \bm{A}_{\beta\times\beta} \odot \bm{B}
\end{equation}
Applying this to \cref{eq:coreg_matrices_example} yields coregionalization matrices that are described by a $\boxtimes$-product between a $(M+1)O\times(M+1)O$ matrix containing only weights, $\bm{W}_v$, and a $O\times O$ matrix $\bm{H}_{v\theta}$ describing task-specific correlations across its objectives:
\begin{equation}
\bm{C}_{v\theta} = \bm{W}_v \boxtimes \bm{H}_{v\theta}
\end{equation}
Careful inspection of $\bm{W}_v$ reveals that it can be written as an outer product, $\bm{w}_v\bm{w}_v^T$, where $\bm{w}_v = (\bm{w}_m^T, \bm{w}_t^T)^T$ and
\begin{equation}
\begin{split}
\bm{w}_m &=
\left(
\ldots, \zerovector{O}, \smash{\underbrace{\onevector{O}}_{\text{m-th entry}}}, \zerovector{O}, \ldots, \zerovector{O}, w_{m1}, w_{m2}, \ldots, w_{mO}
\right)^T \\ \\
\bm{w}_t &= 
\left(
\ldots, \zerovector{O}, \onevector{O}
\right)^T,
\end{split}
\end{equation}
implying that $\bm{W}_v$ is \ac{PSD} for all $w_{mo} \in \mathbb{R}$. Inserting this into \cref{eq:multitask_kernel} yields the final expression for our kernel, which we write in matrix form of size $(M+1)O \times (M+1)O$ for convenience:
\begin{equation}
\begin{split}
\bm{K}_\ourmethod(\bm{x},\bm{x'}) &= \sum_{v\in\mathcal{M}\cup\{t\}}\bm{w}_v\bm{w}_v^T \boxtimes \bm{K}_v(\bm{x},\bm{x'}), \\
\bm{K}_v(\bm{x},\bm{x'}) &= \sum_{\theta\in\mathcal{O}}\bm{H}_{v\theta}k_{v\theta}(\bm{x},\bm{x'}),
\end{split}
\label{eq:kernel_joint_final}
\end{equation}
where $\bm{K}_v(\bm{x},\bm{x'})$ is a kernel of size $O\times O$ that acts on the objectives of task $v$. Since $\bm{K}_v(\bm{x},\bm{x'})$ and $\bm{W}_v$ are both \ac{PSD}, the kernel in \cref{eq:kernel_joint_final} is also \ac{PSD} and is therefore a valid kernel. Expressing the kernel in \cref{eq:kernel_joint_final} in index notation concludes the proof of \cref{lemma:joint_smog_kernel}.
\end{proof}

\section{\ourmethod target-task prior}
\label{app:target_prior}
In the following we prove \cref{th:prior_target}, which we restate below.
\targetprior*
\begin{proof}
According to \cref{eq:kernel_joint_final_index}, the joint prior model for the meta-observations $\bm{Y}_{\mathrm{meta}}=\big(\bm{Y}_1^\top,\ldots,\bm{Y}_M^\top\big)^\top$ and the target-task function value at a query point $(\x,o)$ given by
    \begin{equation}
        \begin{bmatrix}
            \bm{Y}_\mathrm{meta} \\
            f_{\test o}(\x)
        \end{bmatrix} \sim \mathcal{N} \left(
        \bm{0}, \begin{bmatrix}
            \bm{K}_\mathrm{meta} & \bm{k}_{\mathrm{meta},o} \\
            \bm{k}_{\mathrm{meta},o}^\mathrm{T} & k_{\test,oo}
        \end{bmatrix}
        \right)
    \end{equation}
    where $\bm{K}_\mathrm{meta}$ is of size $N_\mathrm{meta}\times N_\mathrm{meta}$, $\bm{k}_{\mathrm{meta},o}$ of size $N_\mathrm{meta}\times 1$, and $k_{t,oo}$ a scalar, and are given by
    \begin{align*}
    \bm{K}_{\mathrm{meta}} &= \diag\left( \kk[1]{\X_1, \X_1} + \diag(\sigma_{11}^2,\ldots,\sigma_{1O}^2)\otimes\bm{I}_{N_1}, \dots \right), \\
    \bm{k}_{\mathrm{meta},o} &= \big( w_{1o} \kk[1]{\X_1, (\x,o)}, \dots, w_{\nummeta o} \kk[\nummeta]{\X_\nummeta, (\x,o)} \big), \\
    k_{\test, oo} &= \left(\kk[\test]{(\x, o), (\x,o')} + \sum_{\meta\in\metaset} w_{\meta o}w_{\meta o'} \kk[\meta]{(\x, o), (\x,o')}\right)\delta_{oo'} \equiv k_{\test,oo'} \delta_{oo'},
    \end{align*}
    and $N_{\mathrm{meta}} = O\sum_{m=1}^M N_m$.
    Conditioning on the metadata $\metadata$ yields
    \begin{equation}
     p( f_{\test o} \mid \metadata, \x) = \mathcal{N}\left( m_{t,\ourmethod}(\x, o), \kk[t,\ourmethod]{(\x,o), (\x,o')} \right),
    \end{equation}
    where the target-task prior mean $m_{t,\ourmethod}$ and covariance $k_{t,\ourmethod}$ are given by the standard Gaussian conditioning rules
    \begin{align*}
        m_{t,\ourmethod}(\x, o) &= \bm{k}_{\mathrm{meta},o}^\mathrm{T} \bm{K}_\mathrm{meta}^{-1} \bm{Y}_\mathrm{meta}, \\
        \kk[\test,\ourmethod]{(\x,o),(\x,o')} &= \mathrm{k}_{\test,oo'} - \bm{k}_{\mathrm{meta},o}^\mathrm{T} \bm{K}_\mathrm{meta}^{-1} \bm{k}_{\mathrm{meta},o'}.
    \end{align*}
    Now since $ \bm{K}_\mathrm{meta}$ is block-diagonal in $m$, we have
    \begin{equation*}
    \bm{K}_\mathrm{meta}^{-1} = \diag\left( (\kk[1]{\X_1, \X_1} + \diag(\sigma_{11}^2,\ldots,\sigma_{1O}^2)\otimes\bm{I}_{N_1})^{-1}, \dots \right),
    \end{equation*}
    so that
    \begin{align*}
        m_{t,\ourmethod}(\x,o) 
        &= \sum_\meta w_{\meta o} \kk[\meta]{(\x,o),\X_\meta} (\kk[\meta]{\X_\meta, \X_\meta} + \diag(\sigma_{\meta 1}^2,\ldots,\sigma_{\meta O}^2)\otimes\bm{I}_{N_\meta})^{-1} \bm{Y}_\meta \\
        &= \sum_\meta w_{\meta o}\hat{m}_{\meta o}(\x),
    \end{align*}
    where $\hat{m}_\meta(\x,o)$ is the per meta-task posterior mean after conditioning on the corresponding data $\mathcal{D}_\meta$. Similarly, for the covariance, we have
    \begin{align*}
        &\kk[t,\ourmethod]{(\x,o), (\x',o')}
        = k_{\test,oo'} - \bm{k}_{\mathrm{meta},o}^\mathrm{T} \bm{K}_\mathrm{meta}^{-1} \bm{k}_{\mathrm{meta},o'}, \\
        &=\kk[\test]{(\x,o),(\x',o')} + \sum_\meta w_{\meta o}w_{\meta o'} \kk[\meta]{(\x,o),(\x',o')} -\sum_\meta w_{\meta o} \kk[\meta]{(\x,o), \X_\meta} \\
        &\quad \times (\kk[\meta]{\X_\meta, \X_\meta} + \diag(\sigma_{\meta 1}^2,\ldots,\sigma_{\meta O}^2)\otimes\bm{I}_{N_\meta})^{-1} w_{\meta o'} \kk[\meta]{\X_\meta, (\x',o')}, \\
        &= \kk[\test]{(\x,o),(\x',o')} + \sum_\meta w_{\meta o}w_{\meta o'} \big(  \kk[\meta]{(\x,o),(\x',o')} - \kk[\meta]{(\x,o), \X_\meta} \\
        &\quad \times (\kk[\meta]{\X_\meta, \X_\meta} + \diag(\sigma_{\meta 1}^2,\ldots,\sigma_{\meta O}^2)\otimes\bm{I}_{N_\meta})^{-1} \kk[\meta]{\X_\meta, (\x',o')}\big), \\
        &= \kk[\test]{(\x,o),(\x',o')} + \sum_\meta w_{\meta o}w_{\meta o'} \hat{k}_\meta[(\x,o),(\x',o')],
    \end{align*}
    where $\hat{k}_\meta$ is the corresponding per meta-task posterior covariance. \qedhere
\end{proof}

\section{Implementation details}

\subsection{Experimental setup details}\label{app:experimental-setup}

\parahead{Compute environment.}
All experiments are launched through dedicated \texttt{SLURM} array scripts for the synthetic and real-world benchmarks and executed inside the same Apptainer containerized software environment.
The launch scripts submit jobs to a shared partition, using the array index as the run seed.
By default, jobs run with an 8-hour wall-clock limit; \ourmethod is allocated 8 CPU cores ($\approx$16\,GB RAM) and baseline methods are allocated 4 CPU cores ($\approx$8\,GB RAM).
Experiments with 4 objectives required more memory than the default allocation: all baselines were raised to 8 CPU cores ($\approx$16\,GB RAM); additionally, the wall-clock limit for 4-objective real-world (terrain) jobs was extended to 16 hours to accommodate longer evaluation times.

\parahead{Metrics.}
For each run and at each iteration, we observe the difference of a configuration's hypervolume to the best-observed hypervolume across models, repetitions, and iterations.
We plot the mean difference (averaged across repetitions) and the \ac{SEM}.

\parahead{Reference point.}
We set the \ac{HV} reference point using BoTorch's \texttt{infer\_reference\_point}\footnote{\url{https://botorch.readthedocs.io/en/v0.16.1/_modules/botorch/utils/multi_objective/hypervolume.html}, accessed on 01/26/2026} applied to normalized Pareto-optimal objective values.
The reference point is chosen slightly \emph{worse} than the Pareto nadir point by moving it back by a fixed fraction (we use BoTorch's default $0.1$, proposed by~\citet{ishibuchi2011many}) of the observed objective range in each dimension.
To improve numerical stability, we standardize the objective values to have zero mean and unit variance before computing the reference point and expected hypervolume.

\parahead{Hyperpriors and marginal-likelihood optimization.}
For \ourmethod's equicorrelated meta-task kernel, we place a scaled $\mathrm{Beta}(2,2)$ hyperprior on the correlation parameter $\rho$ over the full admissible interval $\left[-\frac{1}{O-1}, 1\right]$, matching the support induced by the equicorrelation constraint while still favoring interior correlation values.
For every marginal-likelihood optimization in \cref{eq:meta_likelihood_optimization,eq:target_likelihood_optimization}, we use three random restarts, with the first restart initialized at the default hyperparameter configuration.

\subsection{Mixed space acquisition function optimization}\label{app:mixed-space-af}

We employ an interleaving scheme to optimize the acquisition function, similar to that used by~\citet{wan2021think} or \citet{papenmeier2023bounce}.
In particular, we optimize over the discrete variables by selecting the point with the maximum acquisition value from $2^{14}$ candidates uniformly sampled from the set of all possible discrete combinations.
We then fix the discrete variables and optimize over the remaining variables using gradient-based acquisition function maximization with two random restarts and 512 initial samples to select the two starting points.
We repeat this interleaving scheme $5$ times, starting with the discrete optimization and a uniformly random initialization of the continuous variables.

For categorical variables, we use one-hot encoding, representing $n$ categories as $n$ distinct problem dimensions.
We exclude invalid configurations (e.g., two categories being set to ``1'') from the set of candidates of the \ac{AF} maximizer.

\subsection{Benchmark details}\label{app:benchmark-details}

\parahead{Sinusoidal benchmark.} 
As a simple sanity check, we evaluate \ourmethod on a one-dimensional benchmark function
\begin{equation*}
f_{1,1}(x)=\sin(x-\delta)\qquad f_{2,1}(x)=\sin(x)\qquad f_{3,1}(x)=\sin(x+\delta)\qquad f_{m,2}(x)=f_{m,1}(x+\phi)
\end{equation*}
where $\delta=\frac{\pi}{12}$ and $\phi=\frac{\pi}{6}$.
The target task is a weighted sum of the source tasks $f_{t,o}(x)=\bm{f}_o(x)^\intercal\bm{w}_o$, where $\bm{f}_o(x)=(f_{1,o}(x),\ldots,f_{3,o}(x))$, $\bm{w}_1=(0.5, 0.35, 0.15)$, and $\bm{w}_2=(0.4, 0.4, 0.2)$.
We do not benchmark on this problem, but it is instructive to visualize the intermediate posterior of \ourmethod (see~\cref{fig:sinusoidal_results}).

\parahead{The adapted \hartmann benchmark.} The \hartmann problem is defined as\footnote{\url{https://www.sfu.ca/~ssurjano/hart6.html}, accessed on: 05/06/2026}
\[
f(\bm{x})=-\sum_{i=1}^4\alpha_i\exp\left(-\sum_{j=1}^6A_{ij}(x_j-P_{ij})^2\right),
\]
where $\bm{\alpha}$ is a vector of length 4 and $\bm{A}$ and $\bm{P}$ are known $4\times 6$ matrices:
\begin{align}
    \bm{A} &= \begin{pmatrix}
        10 & 3 & 17 & 3.5 & 1.7 & 8 \\
        0.05 & 10 & 17 & 0.1 & 8 & 14\\
        3 & 3.5 & 1.7 & 10 & 17 & 8\\
        17 & 8 & 0.05 & 10 & 0.1 & 14
    \end{pmatrix}\\
    \bm{P}&=\frac{1}{1000}\begin{pmatrix}
        1312 & 1696 & 5569 & 124 & 8283 & 5886 \\
        2329 & 4135 & 8307 & 3736 & 1004 & 9991 \\
        2348 & 1451 & 3522 & 2883 & 3047 & 6650 \\
        4047 & 8828 & 8732 & 5743 & 1091 & 381
    \end{pmatrix}
\end{align}
We adapt this benchmark to allow for meaningful multi-task, multi-objective optimization.
First, we sample a separate coefficient vector, $\bm{\alpha}_m$, for each task.
Second, we define an offset vector $\bm{\varepsilon}_o$ per objective and, for each combination of task and objective $(m,o)\in\mathcal{M}^*\times\mathcal{O}$, aim to minimize
\[
f_{m,o}(\bm{x})=-\sum_{i=1}^4\alpha_{m,i}\exp\left(-\sum_{j=1}^6A_{ij}(x_j-P_{ij}-\varepsilon_{o,j})^2\right).
\]
We sample $\bm{\alpha}_m$ and $\bm{\varepsilon}_o$ as follows:
\begin{align}
    \alpha_{m,1}\sim\mathcal{U}(1.0,1.02)\\
    \alpha_{m,2}\sim\mathcal{U}(1.18,1.2)\\
    \alpha_{m,3}\sim\mathcal{U}(2.0,3.0)\\
    \alpha_{m,4}\sim\mathcal{U}(3.2,3.4)\\
    \varepsilon_{o,j}\sim\mathcal{U}(0,0.15)\\
\end{align}

The distribution for $\varepsilon_{o,j}$ is chosen such that the Gaussian blobs still overlap.
The minimum full width at half maximum of any Gaussian blob is given by $2\sqrt{\frac{\ln 2}{17}}\approx 0.4$, where $17$ is the largest element in $\bm{A}$.
Hence, a maximum offset of $0.15$ is sufficient to ensure correlation and a non-trivial Pareto front between the objectives.

\parahead{The terrain benchmark.}
We rely on the MetaBox implementation by~\citet{ma2025metabox}, which implements the Terrain benchmark defined in~\citet{shehadeh2025benchmarking}.
While MetaBox provides an additional path-clearance cost, we use only the four defined in~\citet{shehadeh2025benchmarking} (path-length cost, obstacle-avoidance cost, altitude cost, and smoothness cost).
Furthermore, we define a variant with two objectives, optimizing both path length cost and obstacle avoidance cost.
We choose target and ``meta'' terrains uniformly at random with a fixed seed.

\parahead{Adapted \branincurrin benchmark}

We define a multi-task, multi-objective variant of the classical \branincurrin test problem on the unit box $\mathcal{X}=[0,1]^2$.
We consider a set of tasks $\nu \in \mathcal{M}^* = \{0,1,\dots,M\}$, where $\nu=0$ denotes the target task and $\nu\in\{1,\dots,M\}$ are meta-tasks, and objectives $o\in\mathcal{O}=[O]$.
For general~$O$ we alternate between Branin-type ($o$ even) and Currin-type ($o$ odd) objectives, recovering the structure of the standard bi-objective \branincurrin benchmark when $O=2$.

To ensure that different objectives attain their optima at different locations (and thus induce non-trivial Pareto fronts), we draw an input-shift vector $\bm{\varepsilon}_o \sim \mathcal{U}([-0.01,0.01]^2)$ for each objective~$o$ and evaluate every objective at the shifted and clamped input:
\begin{equation}
  \tilde{\bm{x}}
  = \Pi_{[0,1]^2}\!\bigl(\bm{x}-\bm{\varepsilon}_o\bigr),
\end{equation}
where $\Pi_{[0,1]^2}$ denotes component-wise clipping to $[0,1]$.

Each task--objective pair $(\nu,o)$ has its own parameter vector $\bm{\theta}_{\nu,o}$.
Rather than sampling all parameters independently, we employ a two-stage scheme that controls inter-task similarity via a \emph{perturbation scale} $\sigma_p\ge 0$ (default $\sigma_p=0.05$).
\begin{enumerate}
  \item \textbf{Target parameters.} For every objective $o$, a target parameterisation $\bm{\theta}_{0,o}$ is drawn from a \emph{tight} range (specified below).
  \item \textbf{Meta-task parameters.} For $\nu\ge 1$ each $\bm{\theta}_{\nu,o}$ is obtained by perturbing $\bm{\theta}_{0,o}$ with controlled noise whose magnitude is governed by $\sigma_p$, and setting $\sigma_p=0$ makes every meta-task identical to the target.
\end{enumerate}
Every pair $(\nu,o)$ uses a deterministic random seed derived from the global seed, so the first $M$ tasks are invariant when the total number of meta-tasks is increased.

\parahead{Branin-type objectives.}
The shifted input $\tilde{\bm{x}}=(\tilde{x}_1,\tilde{x}_2)$ is mapped to the standard Branin domain via $z_1=15\,\tilde{x}_1-5$, $z_2=15\,\tilde{x}_2$:
\begin{equation}
  B_{\nu,o}(\bm{x})
  = \bigl(z_2 - b_{\nu,o}\,z_1^2 + c_{\nu,o}\,z_1
      - r_{\nu,o}\bigr)^{2}
    + s(1-t)\cos(z_1) + s,
  \qquad s=10,\;t=\tfrac{1}{8\pi}.
\end{equation}
The target parameters are sampled from tight ranges:
\begin{equation}\label{eq:branin-tight}
  b_{0,o}\sim\mathcal{U}(0.125,\,0.133),\quad
  c_{0,o}\sim\mathcal{U}(1.55,\,1.63),\quad
  r_{0,o}\sim\mathcal{U}(5.9,\,6.1),
\end{equation}
These ranges are centered on the standard Branin values $b^*\!\approx0.129$, $c^*\!\approx1.59$, $r^*\!=6$.
Meta-task parameters are obtained by additive Gaussian perturbation, clipped to the same tight bounds $[l_k,u_k]$:
\begin{equation}
  \theta_{\nu,o,k}
  = \operatorname{clip}\!\bigl(
      \theta_{0,o,k} + \delta_k,\;l_k,\;u_k
    \bigr),
  \qquad
  \delta_k \sim \mathcal{N}\!\bigl(0,\,
    [\sigma_p(u_k-l_k)]^{2}\bigr),
  \quad k\in\{b,c,r\}.
\end{equation}

\parahead{Currin-type objectives.}
The Currin rational form is~\citep{currin1991bayesian}:
\begin{equation}
  C_{\nu,o}(\bm{x})
  = \Bigl(1-\exp\!\bigl(-\tfrac{1}{2\tilde{x}_2}\bigr)\Bigr)\,
    \frac{p_{\nu,o}(\tilde{x}_1)}{q_{\nu,o}(\tilde{x}_1)},
\end{equation}
where $p_{\nu,o}$ and $q_{\nu,o}$ are order-three polynomials that are fixed in \citet{currin1991bayesian}, but will be slightly varied for the purpose of multi-objective meta-learning.
We set $\tilde{x}_2\leftarrow\max(\tilde{x}_2,10^{-6})$ for numerical stability.
The numerator and denominator polynomials are
\begin{align}
  p_{\nu,o}(u)
    &= a^{(n)}_{\nu,o,1}\,u^3
     + a^{(n)}_{\nu,o,2}\,u^2
     + a^{(n)}_{\nu,o,3}\,u
     + a^{(n)}_{\nu,o,4},\\
  q_{\nu,o}(u)
    &= a^{(d)}_{\nu,o,1}\,u^3
     + a^{(d)}_{\nu,o,2}\,u^2
     + a^{(d)}_{\nu,o,3}\,u
     + a^{(d)}_{\nu,o,4}.
\end{align}
The target coefficients are obtained by mildly perturbing the standard Currin constants with i.i.d.\ uniform multiplicative noise:
\begin{align}
  \bm{a}^{(n)}_{0,o}
    &= (2300,\,1900,\,2092,\,60)\;\circ\;\bm{\eta}^{(n)}_{0,o},
  &\eta^{(n)}_{0,o,i}&\sim\mathcal{U}(0.97,\,1.03),\\
  \bm{a}^{(d)}_{0,o}
    &= (100,\,500,\,4,\,20)\;\circ\;\bm{\eta}^{(d)}_{0,o},
   &\eta^{(d)}_{0,o,j}&\sim\mathcal{U}(0.97,\,1.03).
\end{align}
Meta-task coefficients are obtained by log-normal multiplicative perturbation of the target values:
\begin{equation}
  a^{(\cdot)}_{\nu,o,i}
  = a^{(\cdot)}_{0,o,i}\,
    \exp(\xi_i),
  \qquad
  \xi_i \sim \mathcal{N}\!\bigl(0,\,(0.1\,\sigma_p)^{2}\bigr),
\end{equation}
This perturbation is applied independently to every numerator and denominator coefficient.

\subparagraph{Negation and normalization.}
We negate the raw function values so that the benchmark is framed as maximization:
\begin{equation}
  f_{\nu,o}(\bm{x})
  = -\begin{cases}
      B_{\nu,o}(\bm{x}), & \text{type}(o)=\text{Branin},\\
      C_{\nu,o}(\bm{x}), & \text{type}(o)=\text{Currin}.
    \end{cases}
\end{equation}
By default, each objective is normalized to $[0,1]$.
We estimate the global minimum~$m_o$ and maximum~$M_o$ of $f_{\nu,o}$ across all tasks $\nu\in\{0,\dots,M\}$ via multi-start L-BFGS-B and return
\begin{equation}
  \hat{f}_{\nu,o}(\bm{x})
  = \operatorname{clip}\!\!\left(
      \frac{f_{\nu,o}(\bm{x})-m_o}{M_o-m_o},\;0,\;1
    \right).
\end{equation}

\subsection{Model details}\label{app:model-details}

In this section, we detail the implementations of the different models used in our experiments.
For \ac{GP}-based methods, we use a Mat\'ern-$\frac{5}{2}$ kernel with \ac{ARD} and a Gamma$(1.5, 1)$ lengthscale hyperprior, a Gaussian likelihood with a LogNormal(-4, 1) hyperprior with a lower bound of $10^{-6}$ on the noise term.
For \ourmethod and \IndSCAML, which train meta-\acp{GP} and only model the difference between the intermediate prior and the function output, we use Mat\'ern-$\frac{5}{2}$ kernels with a LogNormal$(0.5, 1.5)$ prior.
The models using multi-task kernels (\ourmethod and \MOGP) learn a global noise term, again with a LogNormal$(-4, 1)$ hyperprior, and can model the function variance as part of their multi-task kernel.
To maintain the same level of flexibility, \IndSCAML and \IndGP include a function variance term with a LogNormal$(-2, 3)$ prior.

Models starting with ``\texttt{Ind.}'' (\IndGP, \IndABLR, \IndRGPE, \IndSCAML, \IndSGPT, \IndGCThreeP, \IndSMOG) are implemented as independent models, i.e., the parameters of the model of output 1 are independent of the parameters of the model of output 2.

\section{Memory and runtime analysis}\label{app:runtime}

\parahead{Runtime.}
We study the empirical runtime of \ourmethod and competitors and show the results in \cref{tab:timing}.
These measurements use the same containerized HPC setup as the main experiments, but with a dedicated controlled allocation for the timing study.
We run each optimizer on shared Intel Cascade Lake nodes with 4 CPU cores and approximately 16\,GB RAM for the 2-objective case, and with 8 CPU cores and approximately 32\,GB RAM for the 4-objective case.
The benchmark is the 6-dimensional Hartmann function with 2 and 4 objectives as described in \cref{app:benchmark-details}.
We initialize each optimizer with 10 observations on the target task and run 5 BO iterations.
For models leveraging metadata, we use 8 meta tasks with 64 observations each.

For two objectives, \ourmethod takes about 3.3 seconds to construct the meta \acp{GP}, 5 seconds to fit the target \ac{GP}, and 1.4 seconds to optimize the \ac{AF}, i.e., the total runtime per \ac{BO} loop is about 6.4 seconds.
For the four objectives, the time to fit the meta \acp{GP} increases to roughly 6.6 seconds, the time to fit the target \ac{GP} to 13.3 seconds, and the time to optimize the \ac{AF} to 3.5 seconds, totaling in about 16.8 seconds per \ac{BO} loop.

Even though on the higher end, \ourmethod's runtime is on the same order of magnitude as the other methods.

\begin{table}[H]
  \centering
  \caption{Wall-clock timing breakdown per optimizer
           (mean $\pm$ std across 50 runs).
           All times in seconds.}
  \label{tab:timing}
  \resizebox{\textwidth}{!}{%
  \begin{tabular}{@{} l c c c c c @{}}
    \toprule
    \multirow{2}{*}{Optimizer}
      & Meta-GP   & Prior        & Target HP        & AF Opt /  & Total /  \\
      & fit          & constr.          & Opt / iter & iter     & iter     \\
    \midrule
    \multicolumn{6}{@{}l}{\textit{2 objectives}} \\[2pt]
    \ourmethod
      & $3.339 \pm 0.276$ & $0.071 \pm 0.004$ & $5.023 \pm 2.560$  & $1.367 \pm 0.198$ & $6.394 \pm 2.539$  \\
    \IndGP
      & ---               & ---               & $0.332 \pm 0.091$  & $0.248 \pm 0.041$ & $0.586 \pm 0.097$  \\
    \IndABLR
      & ---               & $0.240 \pm 0.015$ & $1.110 \pm 0.055$  & $0.165 \pm 0.017$ & $1.279 \pm 0.061$  \\
    \IndGCThreeP
      & ---               & ---               & $0.154 \pm 0.013$  & $0.903 \pm 0.131$ & $1.061 \pm 0.133$  \\
    \MOGP
      & ---               & ---               & $1.358 \pm 0.468$  & $0.554 \pm 0.090$ & $1.917 \pm 0.452$  \\
    \IndRGPE
      & $6.361 \pm 0.550$ & $0.090 \pm 0.017$ & $0.377 \pm 0.057$  & $0.721 \pm 0.101$ & $1.102 \pm 0.112$  \\
    \IndSGPT
      & $6.330 \pm 0.512$ & $0.091 \pm 0.011$ & $0.344 \pm 0.047$  & $0.713 \pm 0.093$ & $1.062 \pm 0.110$  \\
    \IndSCAML
      & $3.397 \pm 0.312$ & $0.077 \pm 0.003$ & $1.548 \pm 0.374$  & $0.771 \pm 0.107$ & $2.323 \pm 0.393$  \\
    \MOTPE
      & ---               & ---               & ---                & $0.040 \pm 0.004$ & $0.047 \pm 0.004$  \\
    \midrule
    \multicolumn{6}{@{}l}{\textit{4 objectives}} \\[2pt]
    \ourmethod
      & $6.579 \pm 0.480$  & $0.128 \pm 0.007$ & $13.308 \pm 7.385$  & $3.506 \pm 0.458$  & $16.823 \pm 7.367$ \\
    \IndGP
      & ---                & ---               & $0.569 \pm 0.219$   & $0.437 \pm 0.090$  & $1.014 \pm 0.256$  \\
    \IndABLR
      & ---                & $0.300 \pm 0.007$ & $3.178 \pm 0.095$   & $0.386 \pm 0.076$  & $3.572 \pm 0.125$  \\
    \IndGCThreeP
      & ---                & ---               & $0.298 \pm 0.019$   & $1.759 \pm 0.269$  & $2.065 \pm 0.271$  \\
    \MOGP
      & ---                & ---               & $5.928 \pm 1.960$   & $0.803 \pm 0.128$  & $6.739 \pm 1.932$  \\
    \IndRGPE
      & $24.155 \pm 1.680$ & $0.133 \pm 0.005$ & $0.965 \pm 0.225$   & $2.846 \pm 0.670$  & $3.820 \pm 0.820$  \\
    \IndSGPT
      & $24.094 \pm 1.701$ & $0.146 \pm 0.017$ & $0.691 \pm 0.058$   & $2.381 \pm 0.332$  & $3.079 \pm 0.349$  \\
    \IndSCAML
      & $7.349 \pm 0.670$  & $0.169 \pm 0.005$ & $8.886 \pm 1.960$   & $3.329 \pm 0.501$  & $12.226 \pm 2.160$ \\
    \MOTPE
      & ---                & ---               & ---                 & $0.086 \pm 0.010$  & $0.093 \pm 0.010$  \\
    \bottomrule
  \end{tabular}%
  }
\end{table}

\parahead{Memory.}
We also measure the memory consumption of \ourmethod on the same \hartmann benchmark.
We initialize each method with $M\in\{1,2,4,8,16,32,64,128\}$ meta tasks, $N_m=64$ observations per meta task, and 10 observations on the target task.
We then run a single \ac{BO} loop and measure memory consumption using \ac{RSS}.
We record both the initial \ac{RSS} (before the start of the \ac{BO} loop) and the peak \ac{RSS} throughout the job's runtime.
The results are shown in \cref{tab:memory}.

We observe that memory consumption scales sublinearly with $M$: doubling $M$ never leads to a doubling of the memory consumption.
Rather, memory consumption remains almost constant for $M\in\{1,2,4\}$ and only scales moderately for higher $M$.
For four objectives at $M{=}128$, the reported peak \ac{RSS} \emph{decreases} relative to $M{=}64$ while the standard deviation increases by an order of magnitude ($\pm 6116$\,MB vs.\ $\pm 1027$\,MB).  
This is consistent with the working set exceeding available physical memory: \texttt{ru\_maxrss} measures peak \emph{resident} pages, so once the OS begins swapping, the metric is capped by physical RAM rather than reflecting true memory demand.  The monotonic growth observed in the single-objective setting, where the footprint remains within physical memory for all values of~$M$, corroborates this interpretation.
This suggests that scaling to $M{=}128$ meta-tasks with four objectives requires either more physical memory or a more memory-efficient surrogate representation.

Overall, memory consumption is of no particular concern for \ourmethod.

\begin{table}[H]
  \centering
  \caption{Memory consumption as a function of the number of meta-tasks $M$ (50 runs per configuration).
           Init RSS is the process's peak RSS after optimizer initialization.
           Job Peak RSS is the true peak over the entire job
           (via \texttt{resource.getrusage}).
           Scaling factors are relative to the job peak RSS.}
  \label{tab:memory}
  \begin{tabular}{@{} r c c c c @{}}
    \toprule
    \multirow{2}{*}{\# Meta Tasks}
      & Init RSS          & Job Peak RSS       & Increase        & Increase vs          \\
      &  (MB)               & (MB)             & vs prev         & smallest $M$         \\
    \midrule
    \multicolumn{5}{@{}l}{\textit{2 objectives}} \\[2pt]
      1   & $1536.908 \pm 1.260$ & $1704.128 \pm  36.733$ & ---          & $1.00\times$ \\
      2   & $1537.602 \pm 1.212$ & $1755.766 \pm  15.233$ & $1.03\times$ & $1.03\times$ \\
      4   & $1537.153 \pm 1.229$ & $1796.080 \pm  22.439$ & $1.02\times$ & $1.05\times$ \\
      8   & $1537.807 \pm 1.055$ & $1879.171 \pm  40.500$ & $1.05\times$ & $1.10\times$ \\
     16   & $1538.878 \pm 1.068$ & $2145.700 \pm  43.452$ & $1.14\times$ & $1.26\times$ \\
     32   & $1539.919 \pm 1.087$ & $2751.825 \pm  17.388$ & $1.28\times$ & $1.61\times$ \\
     64   & $1543.418 \pm 1.149$ & $3903.600 \pm  43.822$ & $1.42\times$ & $2.29\times$ \\
    128   & $1550.982 \pm 1.107$ & $6187.910 \pm 112.023$ & $1.59\times$ & $3.63\times$ \\
    \midrule
    \multicolumn{5}{@{}l}{\textit{4 objectives}} \\[2pt]
      1   & $1537.091 \pm 1.109$ & $1921.185 \pm   98.827$ & ---            & $1.00\times$ \\
      2   & $1537.175 \pm 1.053$ & $1985.962 \pm  103.897$ & $1.03\times$   & $1.03\times$ \\
      4   & $1537.642 \pm 1.247$ & $2101.813 \pm   94.758$ & $1.06\times$   & $1.09\times$ \\
      8   & $1538.116 \pm 1.023$ & $2517.527 \pm   65.971$ & $1.20\times$   & $1.31\times$ \\
     16   & $1539.082 \pm 1.229$ & $3410.993 \pm   85.776$ & $1.35\times$   & $1.78\times$ \\
     32   & $1540.642 \pm 1.276$ & $5176.973 \pm  104.131$ & $1.52\times$   & $2.69\times$ \\
     64   & $1544.187 \pm 1.095$ & $8607.378 \pm 1027.369$ & $1.66\times$   & $4.48\times$ \\
    128 & $1552.841 \pm 1.140$ & $5644.574 \pm 6115.872$ & $0.66\times$ & $2.94\times$ \\
    \bottomrule
  \end{tabular}%
\end{table}

\section{Ablation studies}\label{app:ablations}

We conduct ablation studies to study the sensitivity of \ourmethod and the other optimizers to different benchmark properties.
We set the default to 2 objectives, 8 meta tasks, and 64 observations per task, and vary one of these parameters per experiment.

In \cref{fig:ablation_n_meta_tasks}, we vary the number of meta tasks and study how this affects the behavior of the different optimizers.
Generally, increasing the number of meta tasks improves performance, particularly in early iterations.
The only exception to this is \IndSGPT, which is only minimally sensitive to the number of meta tasks.
Among the different optimizers, \ourmethod shows the best performance across all numbers of meta tasks.

\begin{figure}[H]
    \centering
    \includegraphics[width=\linewidth]{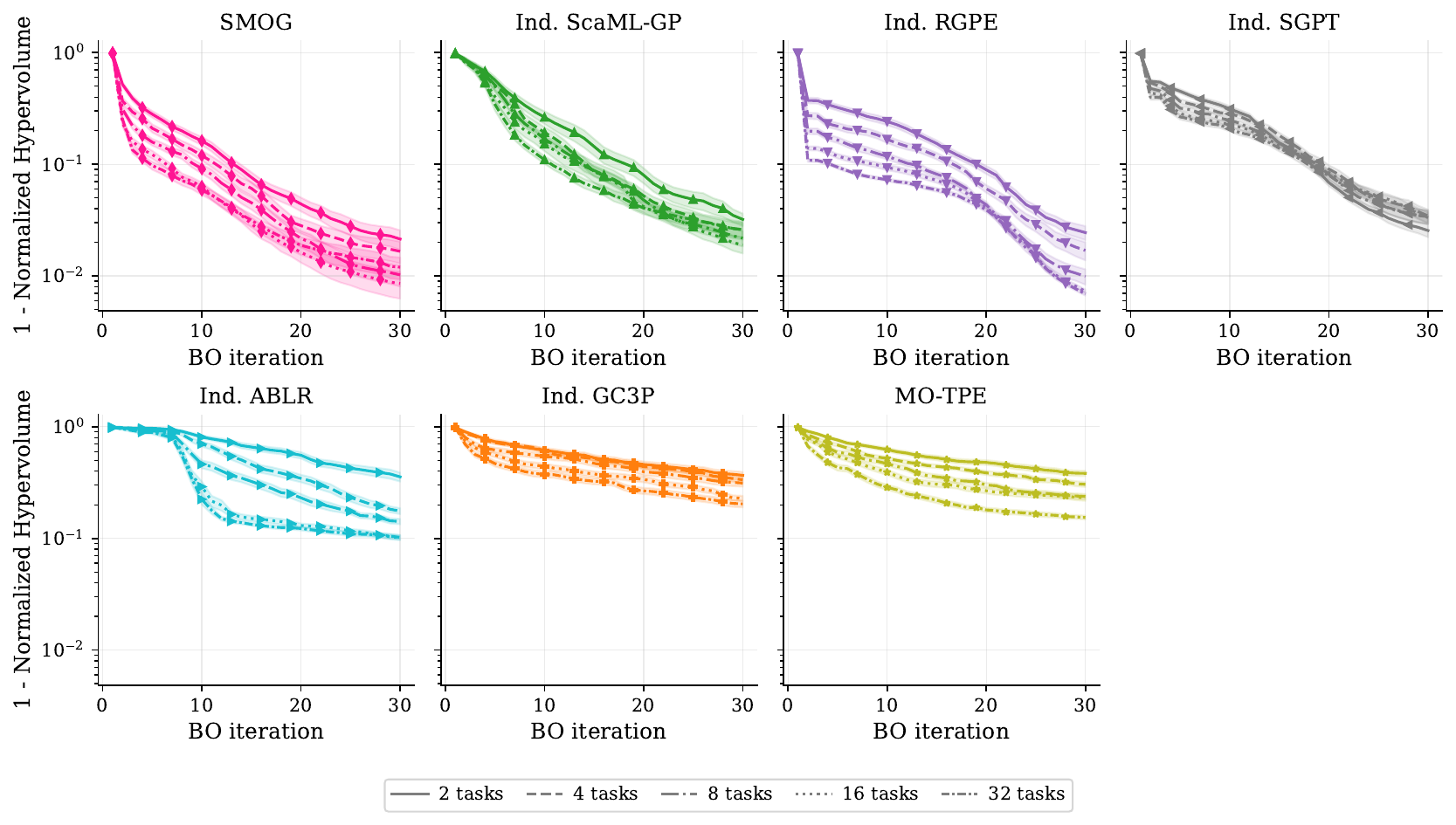}
    \caption{Ablation over the number of meta tasks.}
    \label{fig:ablation_n_meta_tasks}
\end{figure}

Next, we study how varying the number of objectives affects optimization performance (see \cref{fig:ablation_n_objectives}).
The relative ordering of methods is largely preserved across 2, 3, and 4 objectives, with \IndGCThreeP and \MOTPE remaining the weakest in all settings.
\ourmethod is the strongest method on the 2- and 3-objective variants; on the 4-objective variant, \IndRGPE closes the gap and performs on par with \ourmethod, while both clearly outperform the remaining baselines.

\begin{figure}[H]
    \centering
    \includegraphics[width=\linewidth]{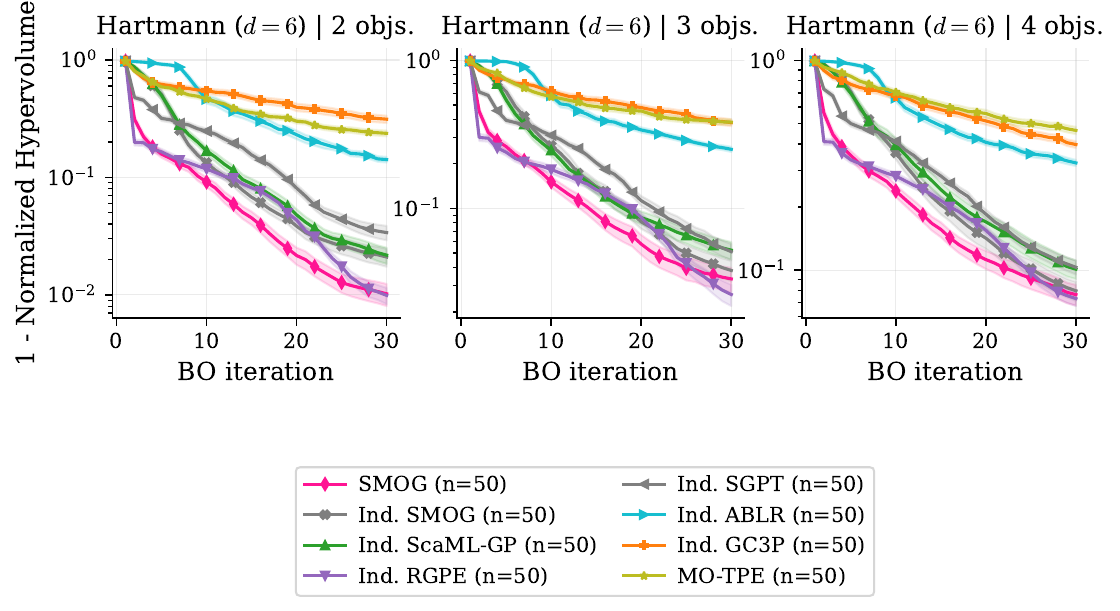}
    \caption{Ablation over the number of objectives.}
    \label{fig:ablation_n_objectives}
\end{figure}

Finally, we study how varying the number of observations per meta task affects the optimizer's performance in \cref{fig:ablation_n_obs}.
As expected, a higher number of observations per task improves optimization performance for all optimizers.
For most methods, the gain from 16 to 32 observations is larger than from 32 to 64, indicating diminishing returns as the metadata grows.
We observe that \ourmethod performs best across all observation counts.

\begin{figure}[H]
    \centering
    \includegraphics[width=\linewidth]{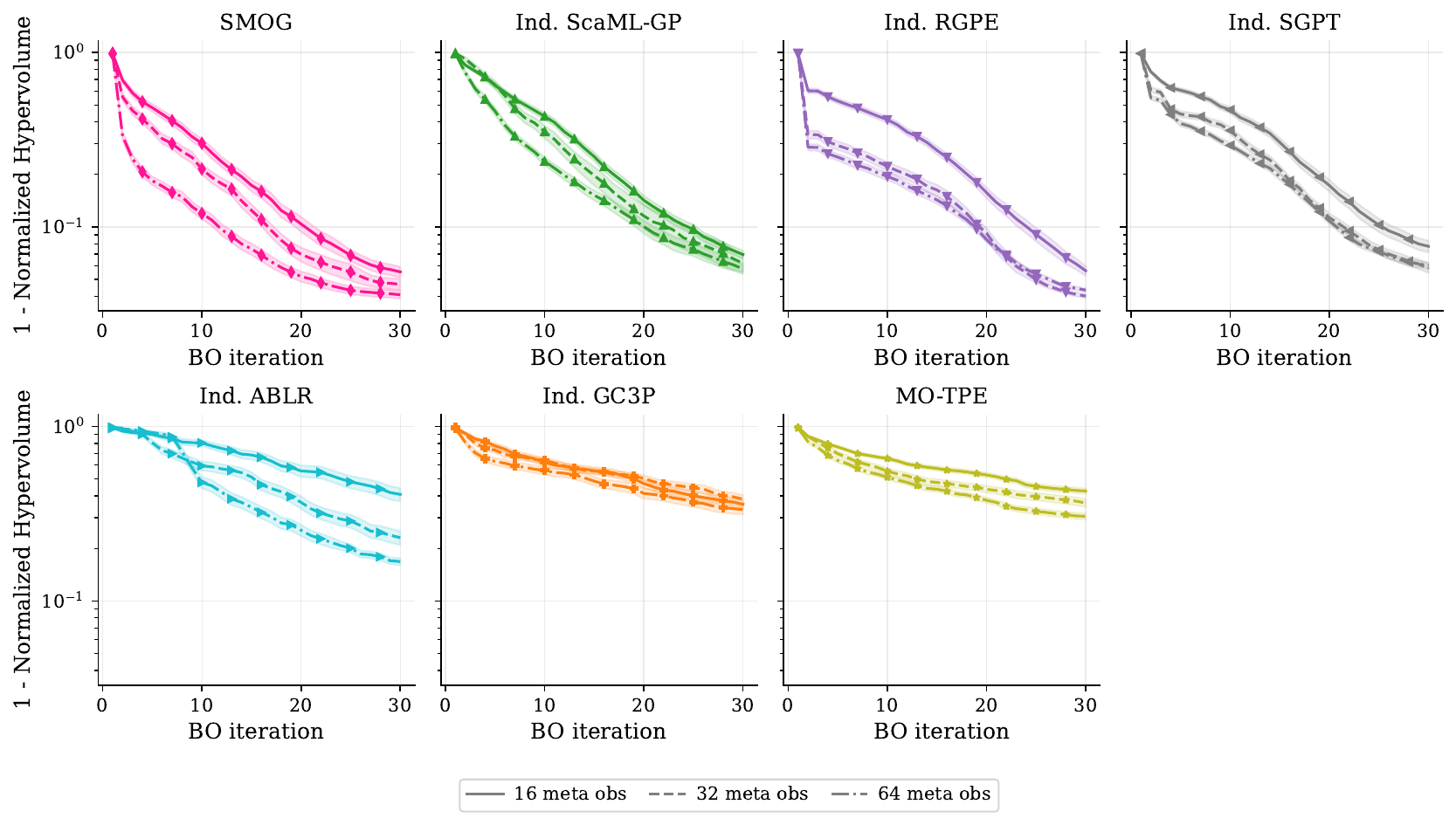}
    \caption{Ablation over the number of observations per meta task.}
    \label{fig:ablation_n_obs}
\end{figure}

\subsection{Acquisition function comparison}\label{app:acqfunc-comparison}
All GP-based optimizers in our benchmark use \texttt{LogEHVI} as the default acquisition function.
A natural question is whether the choice of acquisition function decisively influences performance, or whether the meta-learning components are the main driver.
To investigate this, we re-run all GP-based optimizers on the \branincurrin and Hartmann\,6 benchmarks (8 meta tasks, 64 observations per task, 2 objectives, 50 seeds) with two acquisition functions: \texttt{LogEHVI}~\citep{ament2023unexpected} (solid lines) and \texttt{qLogParEGO}~\citep{ament2023unexpected} (dashed lines).
\MOTPE is excluded from this comparison as it does not employ a GP-based acquisition function.
\Cref{fig:acqfunc_comparison_branincurrin,fig:acqfunc_comparison_hartmann6} show the results.

Across both benchmarks and nearly all surrogates, \texttt{LogEHVI} reaches a lower final \ac{HV} gap than \texttt{qLogParEGO}.
This is consistent with the fact that \texttt{LogEHVI} directly optimizes a hypervolume-based criterion, whereas \texttt{qLogParEGO} relies on random scalarizations of the objectives.
The magnitude of this gap depends on both the surrogate and the benchmark.
On \branincurrin, the two acquisition functions yield largely overlapping curves for several methods (e.g.\ \texttt{Ind.-ScaML-GP}, \texttt{Ind.-RGPE}, \texttt{Ind.-GC3P}), while more pronounced differences appear for \texttt{Ind.-ABLR}.
On Hartmann\,6, \texttt{LogEHVI} provides a clearer advantage for most methods---including \ourmethod, \texttt{Ind.-ScaML-GP}, and \texttt{Ind.-RGPE}---suggesting that hypervolume-aware acquisition is particularly beneficial in higher-dimensional settings.

Importantly, the relative ranking of meta-learning methods is largely preserved under both acquisition functions, and the gap between \texttt{LogEHVI} and \texttt{qLogParEGO} on a given surrogate is typically smaller than the gap between different surrogates.
\ourmethod remains among the top-performing methods on both benchmarks regardless of the acquisition function used, confirming that its benefit stems from the structured task-covariance prior rather than from the choice of acquisition function, and supporting the broader conclusion that the meta-learning component---rather than the scalarization strategy---is the primary source of performance differences across optimizers.

\begin{figure}[H]
    \centering
    \includegraphics[width=\linewidth]{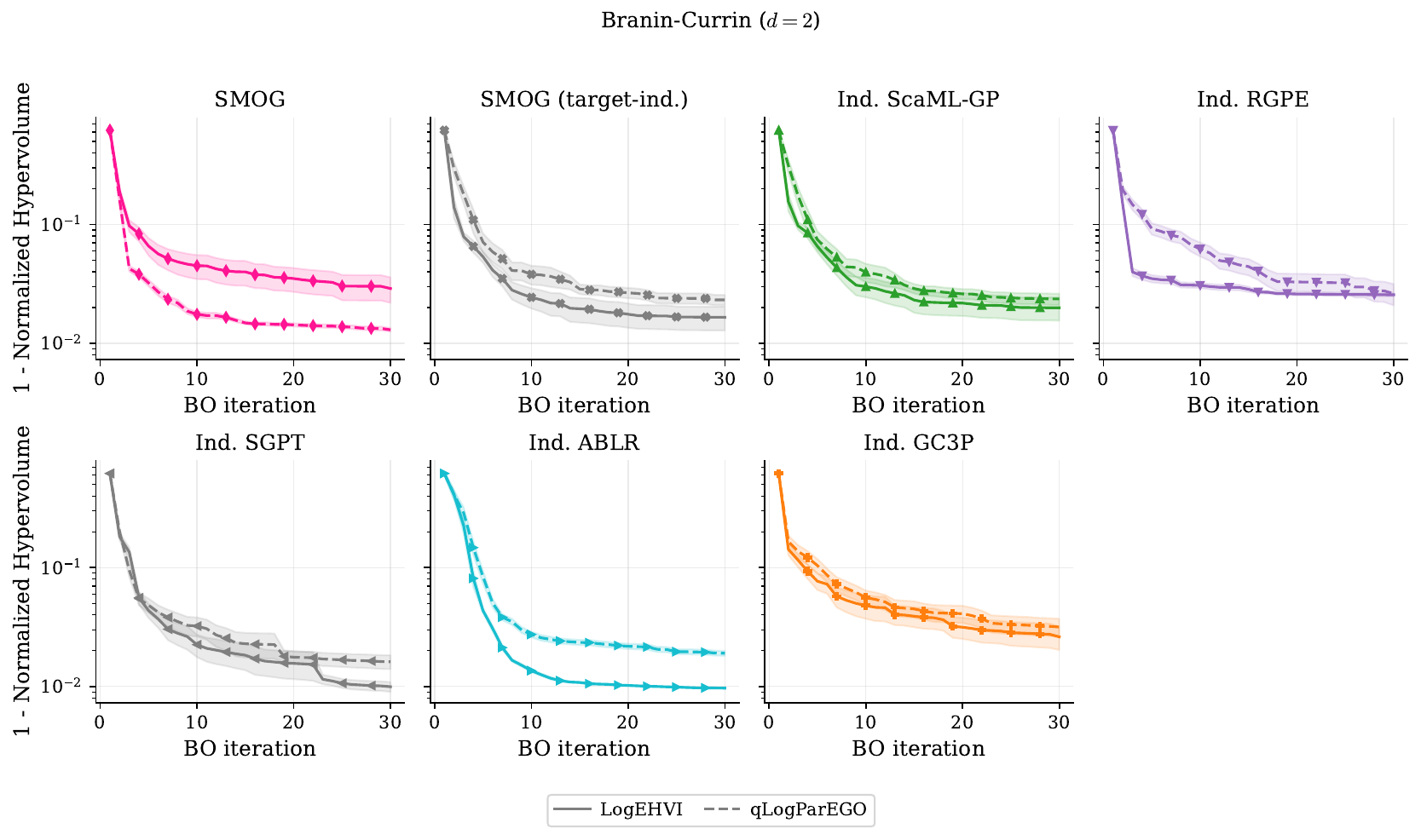}
    \caption{%
        Acquisition function comparison on \branincurrin ($d=2$, 2 objectives).
        Solid lines: \texttt{LogEHVI}; dashed lines: \texttt{qLogParEGO}.
        Each curve shows the mean $\pm$ \ac{SEM} over 50 seeds.
        Lower is better.}
    \label{fig:acqfunc_comparison_branincurrin}
\end{figure}

\begin{figure}[H]
    \centering
    \includegraphics[width=\linewidth]{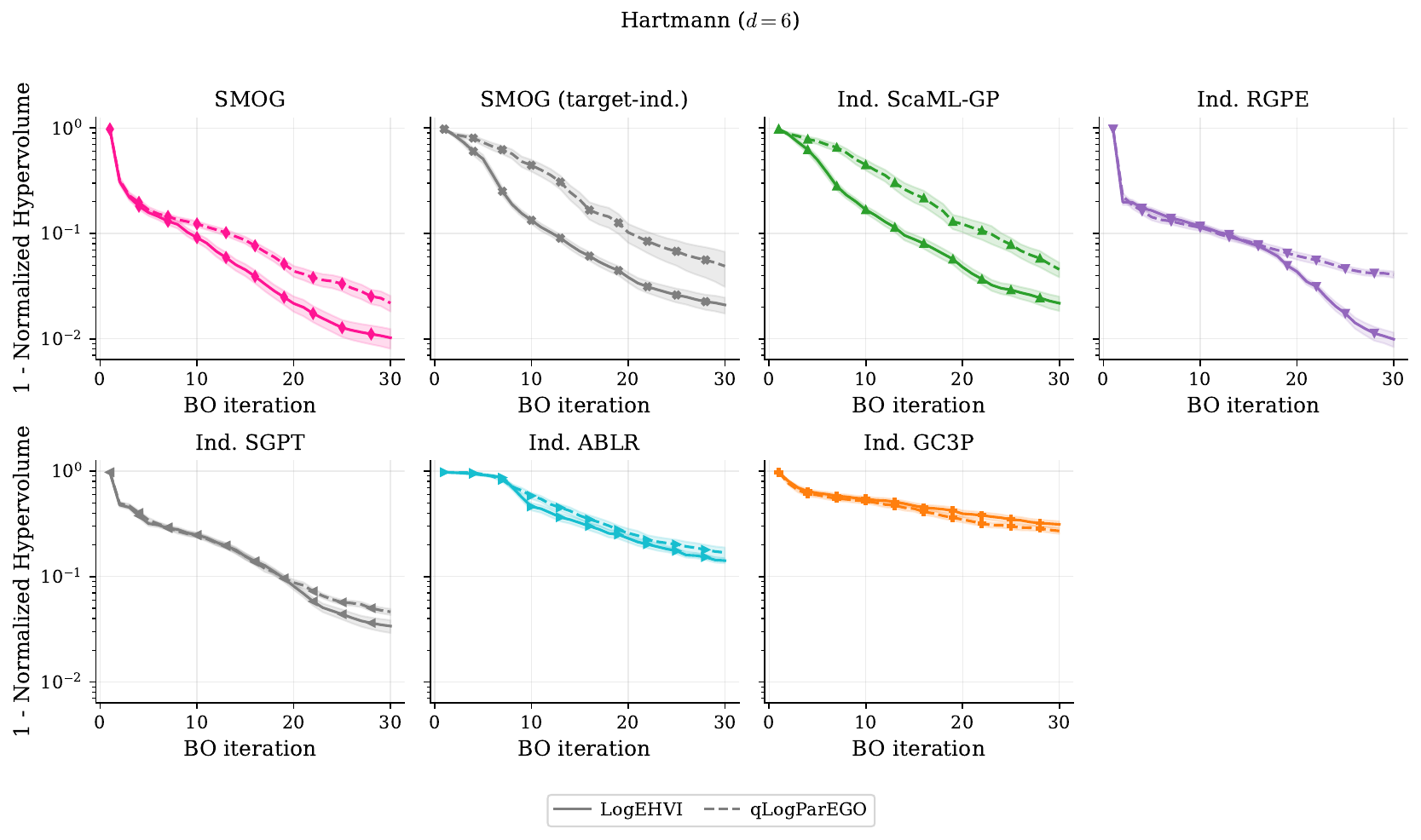}
    \caption{%
        Acquisition function comparison on Hartmann\,6 ($d=6$, 2 objectives).
        Solid lines: \texttt{LogEHVI}; dashed lines: \texttt{qLogParEGO}.
        Each curve shows the mean $\pm$ \ac{SEM} over 50 seeds.
        Lower is better.}
    \label{fig:acqfunc_comparison_hartmann6}
\end{figure}

\section{Additional plots}
\subsection{Synthetic benchmarks}

\Cref{fig:synthetic_results} in the main text shows the cumulative \ac{HV} regret on the synthetic benchmarks.
\Cref{fig:synthetic_hv_appendix} below shows the corresponding \ac{HV} gap.

\begin{figure*}[h]
    \centering
    \includegraphics[width=\linewidth]{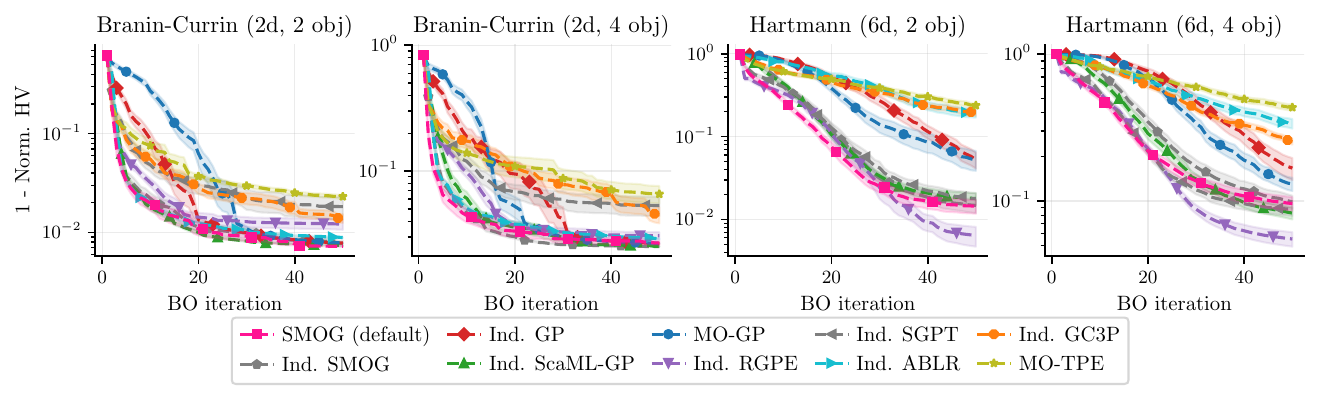}
    \caption{%
        \ac{HV} gap ($1 -$ normalized \ac{HV}; lower is better) on the synthetic benchmarks.
        Columns correspond to benchmark and objective-count combinations:
        BraninCurrin~(2\,obj), BraninCurrin~(4\,obj), \hartmann~(2\,obj), \hartmann~(4\,obj).
        The $y$-axis is on a logarithmic scale.
        All results use 8 meta-tasks and 16 observations per meta-task.%
    }
    \label{fig:synthetic_hv_appendix}
\end{figure*}

\subsection{HPOBench benchmarks}

In this section, we present the \ac{HV} gap results on the individual datasets of the HPOBench benchmark, complementing the cumulative \ac{HV} regret shown in \cref{fig:fcnet_results} in the main text.

\begin{figure}[H]
    \centering
    \includegraphics[width=\linewidth]{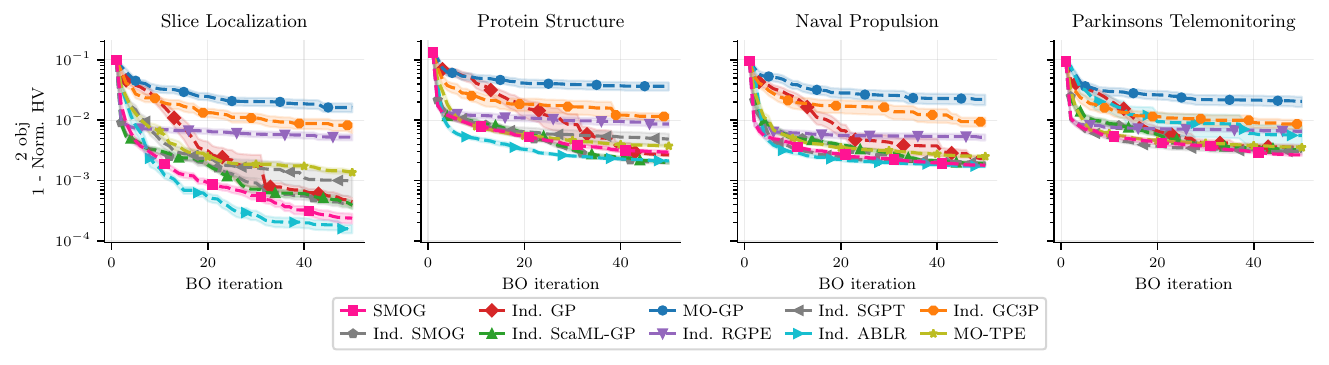}
    \caption{%
        \ac{HV} gap (lower is better) on the individual HPOBench datasets.
        Columns correspond to the four target datasets.
        All results use 3 meta-tasks and 16 observations per meta-task.%
    }
    \label{fig:hpobench_appendix_results}
\end{figure}

\subsection{Terrain benchmarks}

\Cref{fig:terrain_hv} shows the \ac{HV} gap on the Terrain benchmark for all three target tasks and both objective configurations, complementing the cumulative \ac{HV} regret shown in \cref{fig:terrain_cumregret} in the main text.

\begin{figure}[H]
    \centering
    \includegraphics[width=\linewidth]{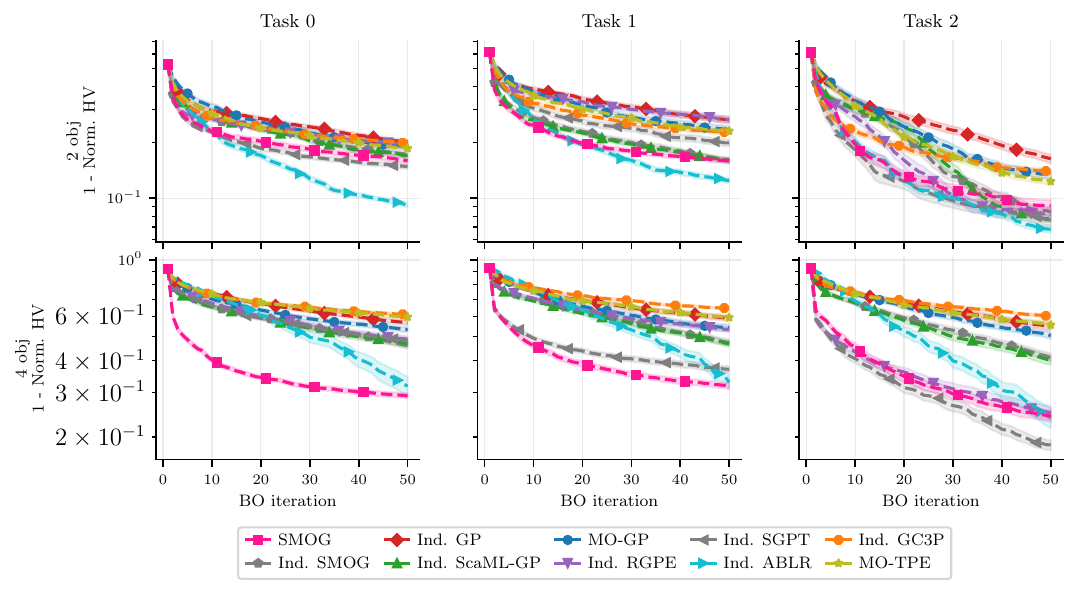}
    \caption{%
        \ac{HV} gap (lower is better) on the Terrain benchmark.
        Rows correspond to 2-objective (top) and 4-objective (bottom) variants;
        columns correspond to target tasks 0, 1, and 2.%
    }
    \label{fig:terrain_hv}
\end{figure}

\subsection{Exemplary \hartmann Pareto front}\label{app:exemplary_hartmann_front}

In this section, we highlight how \ourmethod learns well-behaved Pareto fronts on the two-objective \hartmann benchmark.
For each optimizer, we study the first 15 observations of the run of median \ac{HV} at iteration 15.
Based on these 15 observations, we plot the Pareto front and show the results in \cref{fig:hartmann6_pareto}.
Additionally, we estimate a ``true'' Pareto front by observing $10^6$ points sampled uniformly at random from the search space $\mathcal{X}$.
Even though \IndSGPT and \IndRGPE can find well-performing solutions, \ourmethod finds more Pareto-optimal solutions, with a larger spread.
Furthermore, the solutions found by \ourmethod are better than the ``true'' Pareto front with only 15 observations on the target task.
The metadata-free optimizers \MOGP and \IndGP perform considerably but still outperform \MOTPE by a wide margin.
Note that some methods may not appear in \cref{fig:hartmann6_pareto}: the plot is clipped to $[2.5, \infty)$ on both axes, so methods whose Pareto-front solutions all fall outside this range are not shown.

\begin{figure}[H]
    \centering
    \includegraphics[width=.5\linewidth]{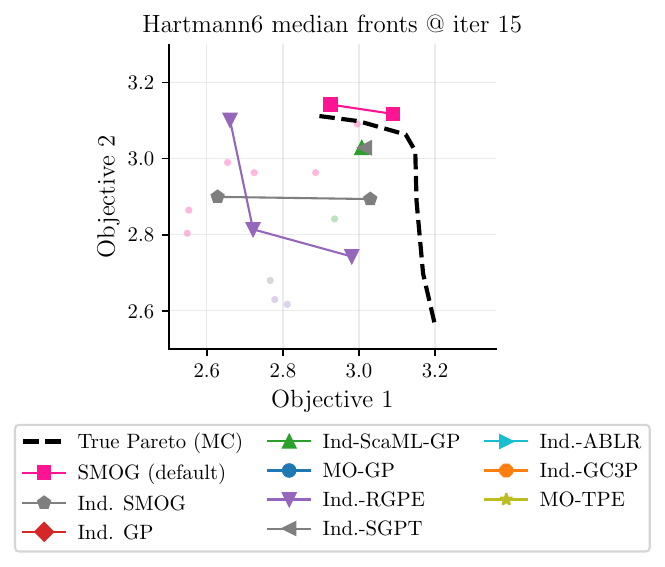}
    \caption{Pareto fronts of the first 15 observations of \ourmethod and competitors on the 6-dimensional two-objective \hartmann benchmark for the run of median \ac{HV} at iteration 15. The ``true'' Pareto front is approximated from $10^6$ uniformly sampled points. Both axes are clipped to $[2.5, \infty)$; methods with no solutions in this range are not shown. \ourmethod learns a Pareto front close to the true Pareto front within 15 iterations.}
    \label{fig:hartmann6_pareto}
\end{figure}

\subsection{HPOBench Pareto fronts}\label{app:pareto-hpobench}

To assess the quality of solutions produced by the different optimizers, we plot the Pareto fronts based on observations across 50 repetitions (see \cref{fig:pareto-hpobench}).
Since HPOBench is a tabular benchmark, one can enumerate all possible solutions and compute the true Pareto front.
However, since the HPOBench dataset provides four observations per configuration and we observe only once, we plot an estimated true Pareto front, which explains why some observations are better than the ``true'' Pareto front.

Importantly, the solutions found by the optimizers are close to the ``true'' Pareto front, showing that the solutions actually are of good quality---a fact that is not evident from only studying the normalized \ac{HV}.

Furthermore, methods that do not leverage metadata (\IndGP and \MOGP) tend to perform worse, confirming the intuition that metadata access improves solution quality.

\begin{figure}[H]
     \centering
     \begin{subfigure}[b]{0.48\textwidth}
         \centering
         \includegraphics[width=\textwidth]{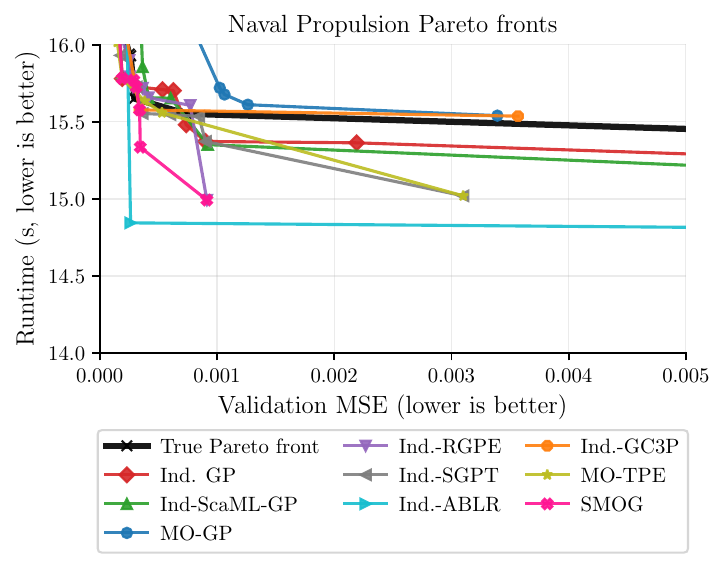}
     \end{subfigure}
     \hfill
     \begin{subfigure}[b]{0.48\textwidth}
         \centering
         \includegraphics[width=\textwidth]{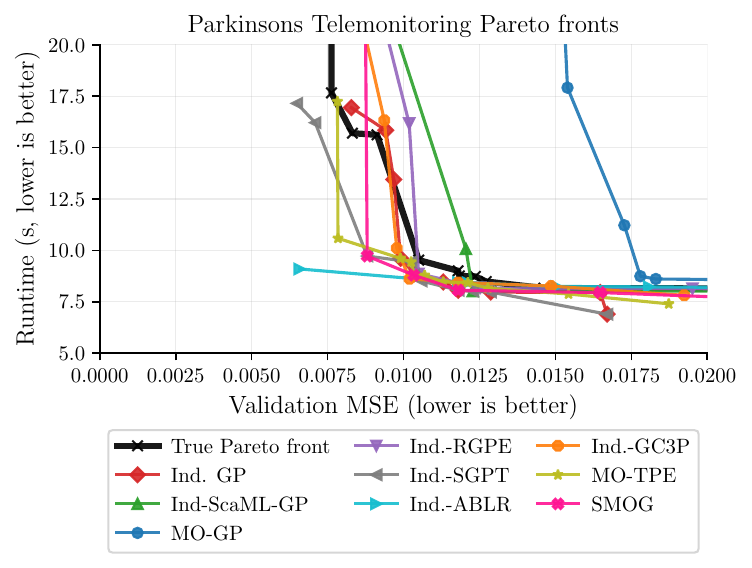}
     \end{subfigure}
     \begin{subfigure}[b]{0.48\textwidth}
         \centering
         \includegraphics[width=\textwidth]{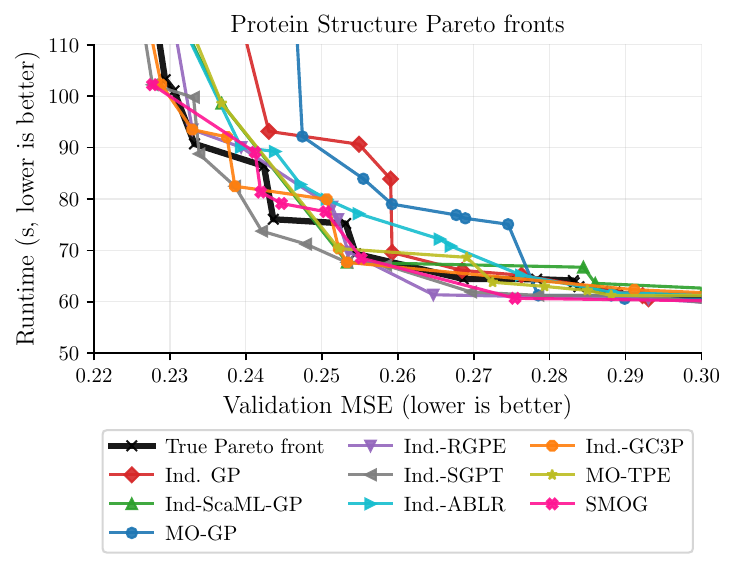}
     \end{subfigure}
     \hfill
     \begin{subfigure}[b]{0.48\textwidth}
         \centering
         \includegraphics[width=\textwidth]{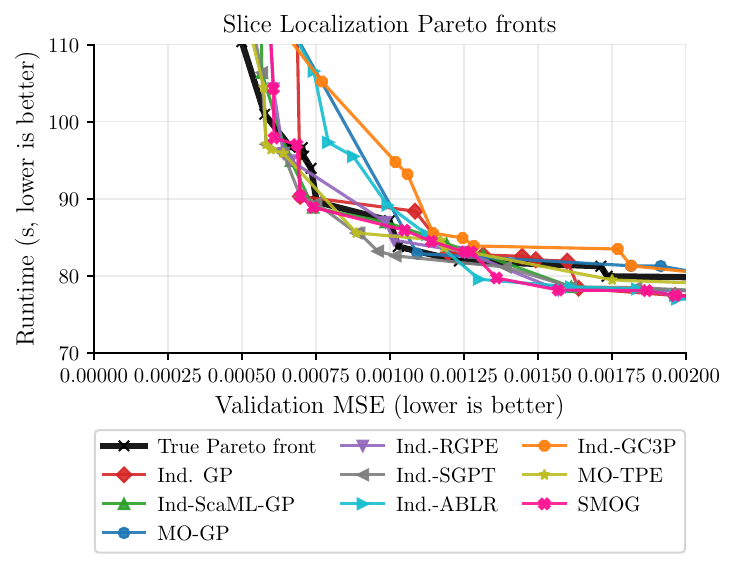}
     \end{subfigure}
    \caption{Pareto fronts of the HPOBench benchmark, pooled across the 50 repetitions. The black line shows an estimated true Pareto front.}
    \label{fig:pareto-hpobench}
\end{figure}

\subsection{Surrogate quality}\label{app:surrogate-quality}

To complement the optimization-performance results in the main text, we directly evaluate the \emph{surrogate quality} of each model:
After fitting on $n_{\mathrm{init}}=10$ target-task observations drawn from a Sobol sequence, we measure the
\acf{NLPD} and \acf{RMSE} against $500$ held-out noiseless Sobol points.
The experiment, therefore, isolates prior and meta-learning quality from the effect of the acquisition function.
Meta-data consists of $M=8$ source tasks with $64$ observations each; we run $50$ independent seeds per setting.
\MOTPE is excluded because it does not maintain a \ac{GP} surrogate.

We report results across three benchmarks (\branincurrin, Hartmann-6, Terrain) and both 2- and 4-objective configurations.
Lower NLPD and RMSE indicate a better-calibrated, more accurate surrogate.
\Cref{fig:surrogate_quality_branincurrin,fig:surrogate_quality_hartmann6,fig:surrogate_quality_terrain}
show the mean $\pm$ standard error of the mean over seeds.

\parahead{RMSE.}
\ourmethod reduces RMSE relative to the non-meta-learning baselines \IndGP and \MOGP most clearly on \branincurrin, where it achieves roughly half their error (e.g.\ $0.073$ vs.\ $0.151$/$0.156$ with 2~objectives).
On Hartmann-6, the gains are modest: with 2~objectives, \ourmethod is slightly better than \IndGP and \MOGP, but with 4~objectives, \ourmethod ($0.432$) is marginally \emph{worse} than \IndGP ($0.406$).
On Terrain, RMSE differences among the leading methods are negligible, and no method clearly dominates.
Across all settings, \IndSCAML tracks \ourmethod very closely in terms of RMSE, indicating that standard multi-output meta-learning already captures most of the accuracy benefit.
Notably, \IndABLR achieves the lowest RMSE on \branincurrin ($0.054$) despite its better-known advantage in calibration, while \IndGCThreeP produces substantially worse \ac{RMSE} on Terrain (exceeding twice the error of the leading methods).

\parahead{NLPD.}
Calibration results are more benchmark-dependent.
On \branincurrin (both 2- and 4-objective), \ourmethod achieves better NLPD than \IndGP and \MOGP but is clearly outperformed by the parametric transfer methods \IndABLR and \IndSGPT, which achieve the best calibration in these settings.
On Hartmann-6 with 2~objectives, \ourmethod ($13.5$) is \emph{worse} than the non-meta-learning baselines \IndGP ($11.4$) and \MOGP ($11.5$), with \IndSGPT ($1.75$) and \IndABLR ($3.12$) far ahead; with 4~objectives the picture improves for \ourmethod ($5.26$), which outperforms \IndGP ($7.79$) and comes close to \IndSGPT ($4.62$) and \IndABLR ($3.54$).
On Terrain, \ourmethod achieves among the best NLPD across both objective counts ($10.3$ and $8.96$, respectively), closely matching \IndSGPT and outperforming \IndABLR and \IndGP.
In striking contrast, \IndGCThreeP produces catastrophically large \ac{NLPD} on Terrain (on the order of $10^9$ and $10^{18}$, respectively), indicating severe miscalibration that is specific to this benchmark.

Taken together, \ourmethod is a reliable choice for predictive accuracy (RMSE) on benchmarks where meta-learning provides a clear signal, but its calibration (NLPD) advantage over non-meta-learning baselines is inconsistent: parametric transfer methods \IndABLR and \IndSGPT achieve better calibration when the target function is smooth and well-matched to the prior (\branincurrin), while \ourmethod's nonparametric approach proves more robust on the structured Terrain benchmark where parametric priors appear to collapse.

\begin{figure}[H]
    \centering
    \begin{subfigure}[b]{0.49\textwidth}
        \centering
        \includegraphics[width=\textwidth]{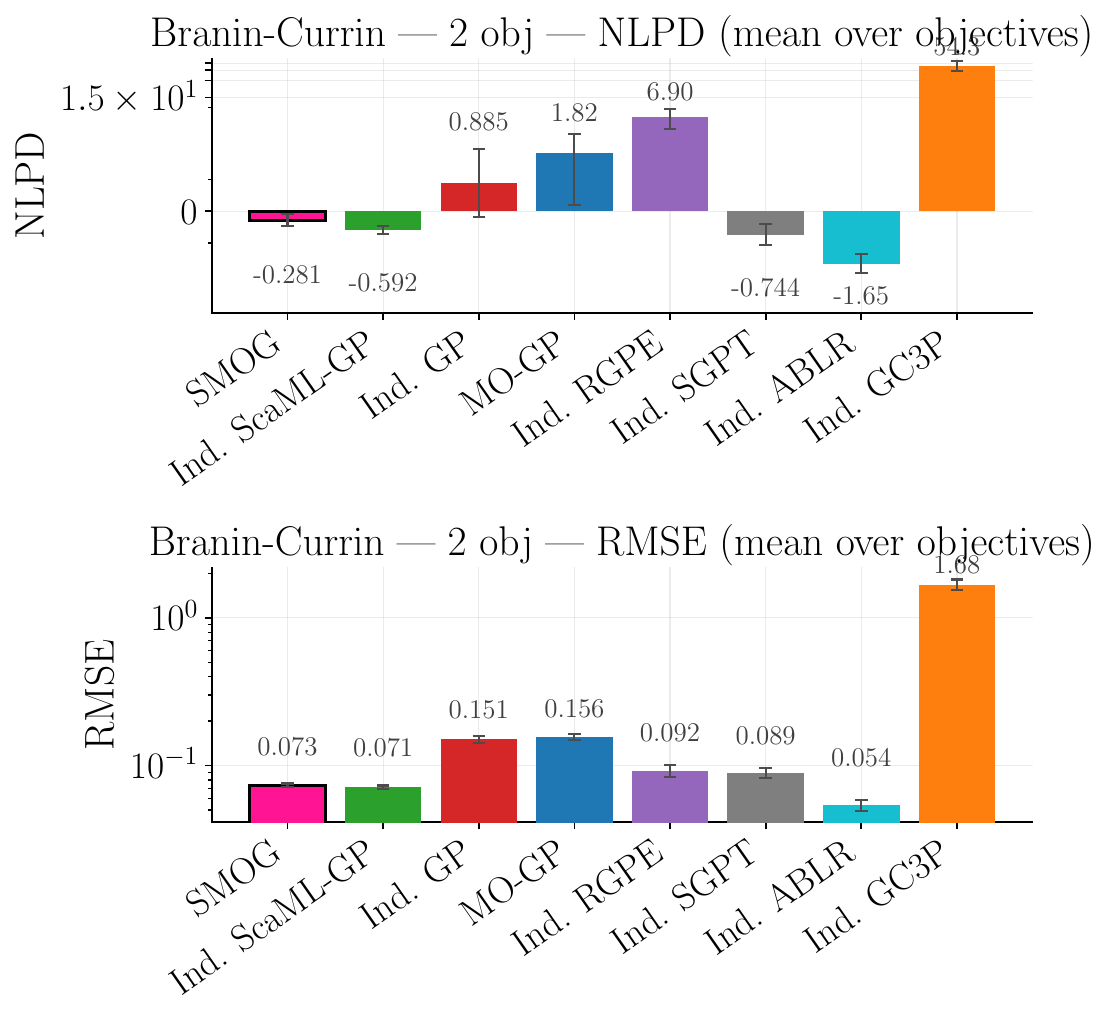}
        \caption{2 objectives}
    \end{subfigure}
    \hfill
    \begin{subfigure}[b]{0.49\textwidth}
        \centering
        \includegraphics[width=\textwidth]{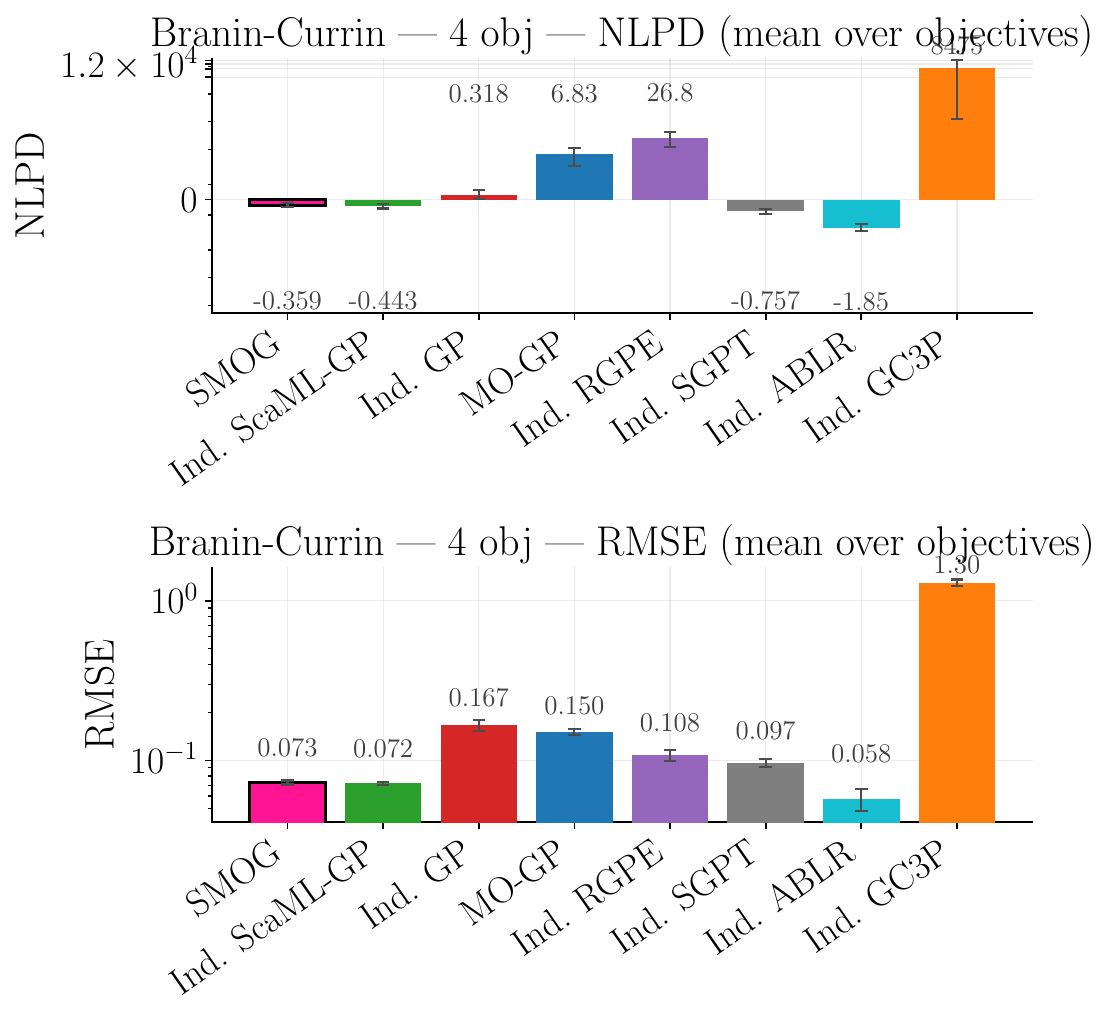}
        \caption{4 objectives}
    \end{subfigure}
    \caption{%
        Surrogate quality (NLPD and RMSE, lower is better) on the \branincurrin benchmark.
        Error bars denote $\pm 1$ \ac{SEM} over seeds.
        \ourmethod improves RMSE over non-meta-learning baselines; NLPD results are mixed, with parametric transfer methods (\IndABLR, \IndSGPT) achieving better calibration than \ourmethod on this benchmark.%
    }
    \label{fig:surrogate_quality_branincurrin}
\end{figure}

\begin{figure}[H]
    \centering
    \begin{subfigure}[b]{0.49\textwidth}
        \centering
        \includegraphics[width=\textwidth]{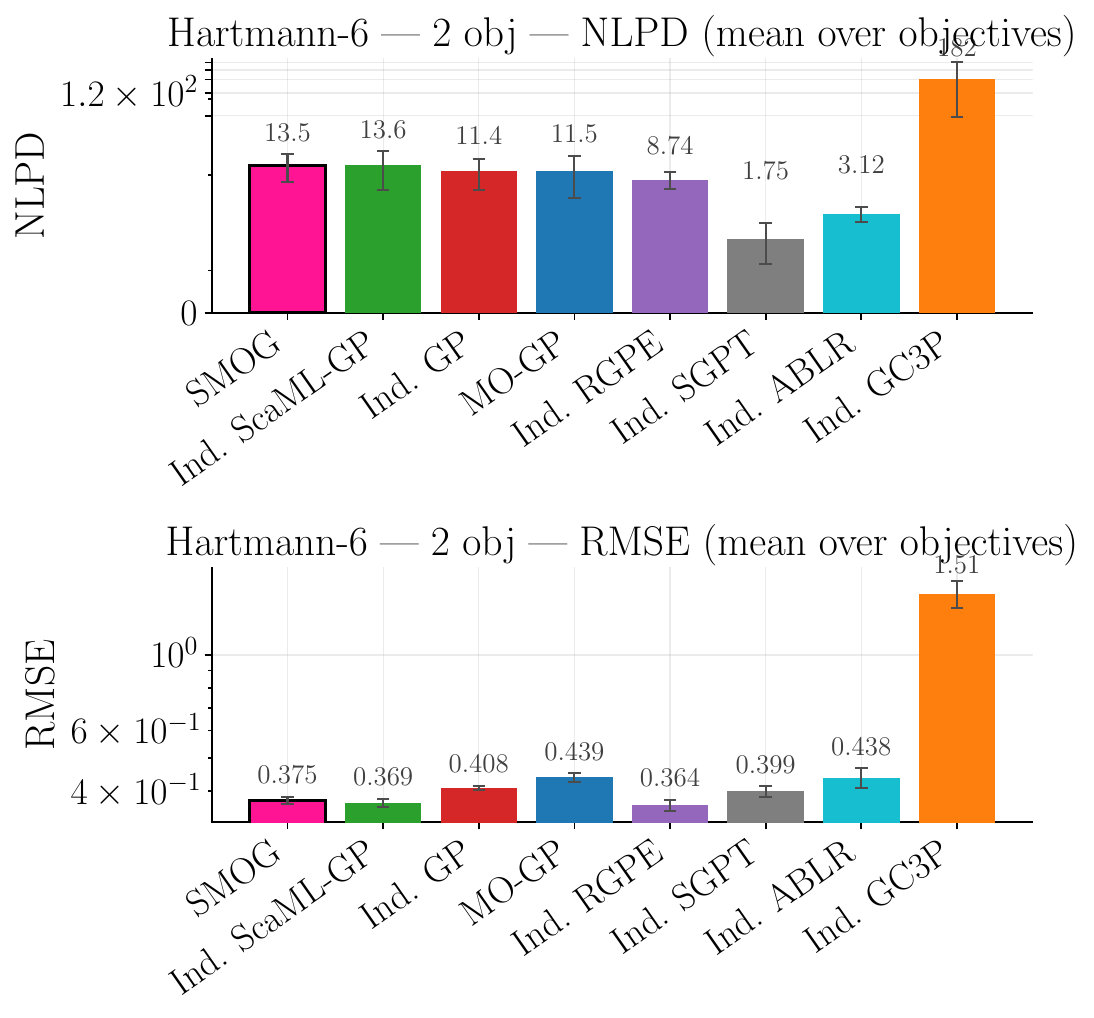}
        \caption{2 objectives}
    \end{subfigure}
    \hfill
    \begin{subfigure}[b]{0.49\textwidth}
        \centering
        \includegraphics[width=\textwidth]{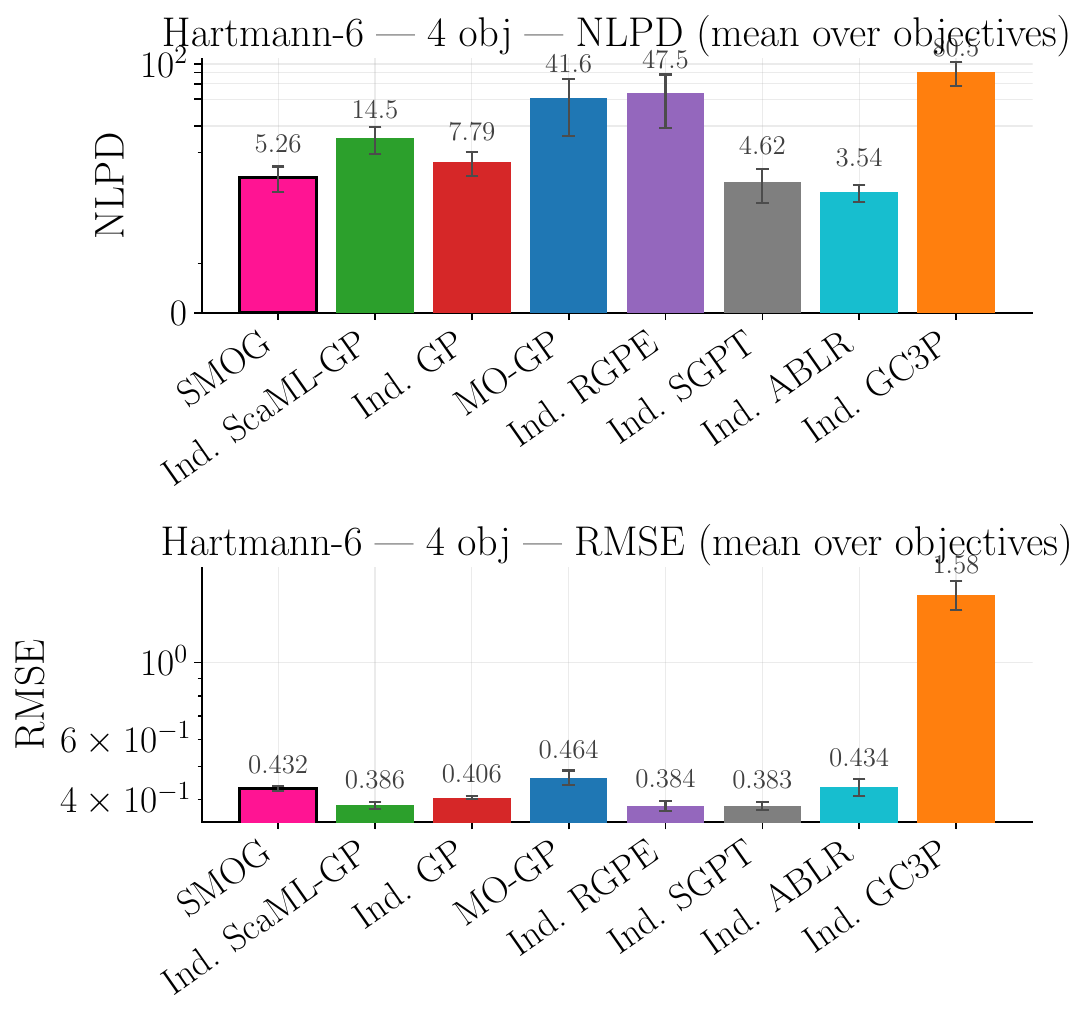}
        \caption{4 objectives}
    \end{subfigure}
    \caption{%
        Surrogate quality (NLPD and RMSE, lower is better) on the Hartmann-6 benchmark.
        Error bars denote $\pm 1$ \ac{SEM} over seeds.%
    }
    \label{fig:surrogate_quality_hartmann6}
\end{figure}

\begin{figure}[H]
    \centering
    \begin{subfigure}[b]{0.49\textwidth}
        \centering
        \includegraphics[width=\textwidth]{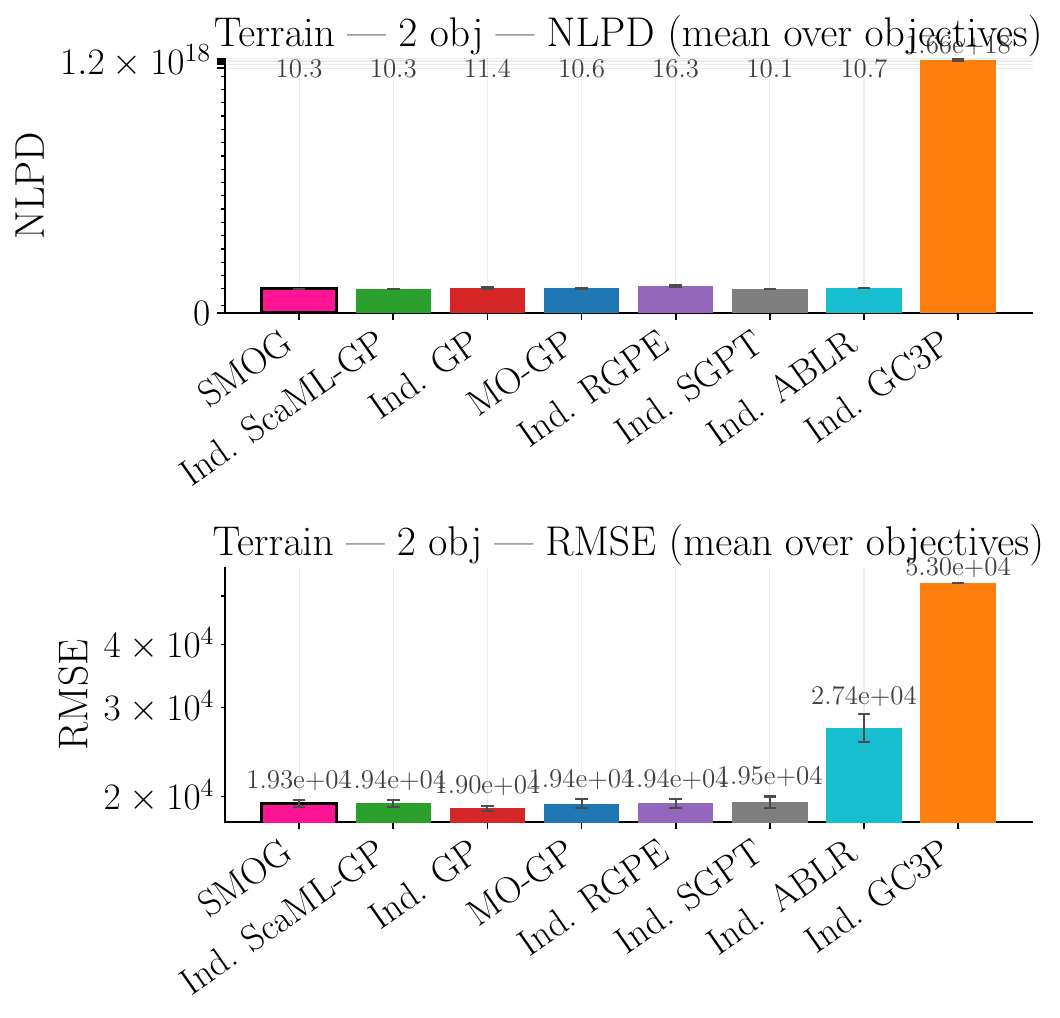}
        \caption{2 objectives}
    \end{subfigure}
    \hfill
    \begin{subfigure}[b]{0.49\textwidth}
        \centering
        \includegraphics[width=\textwidth]{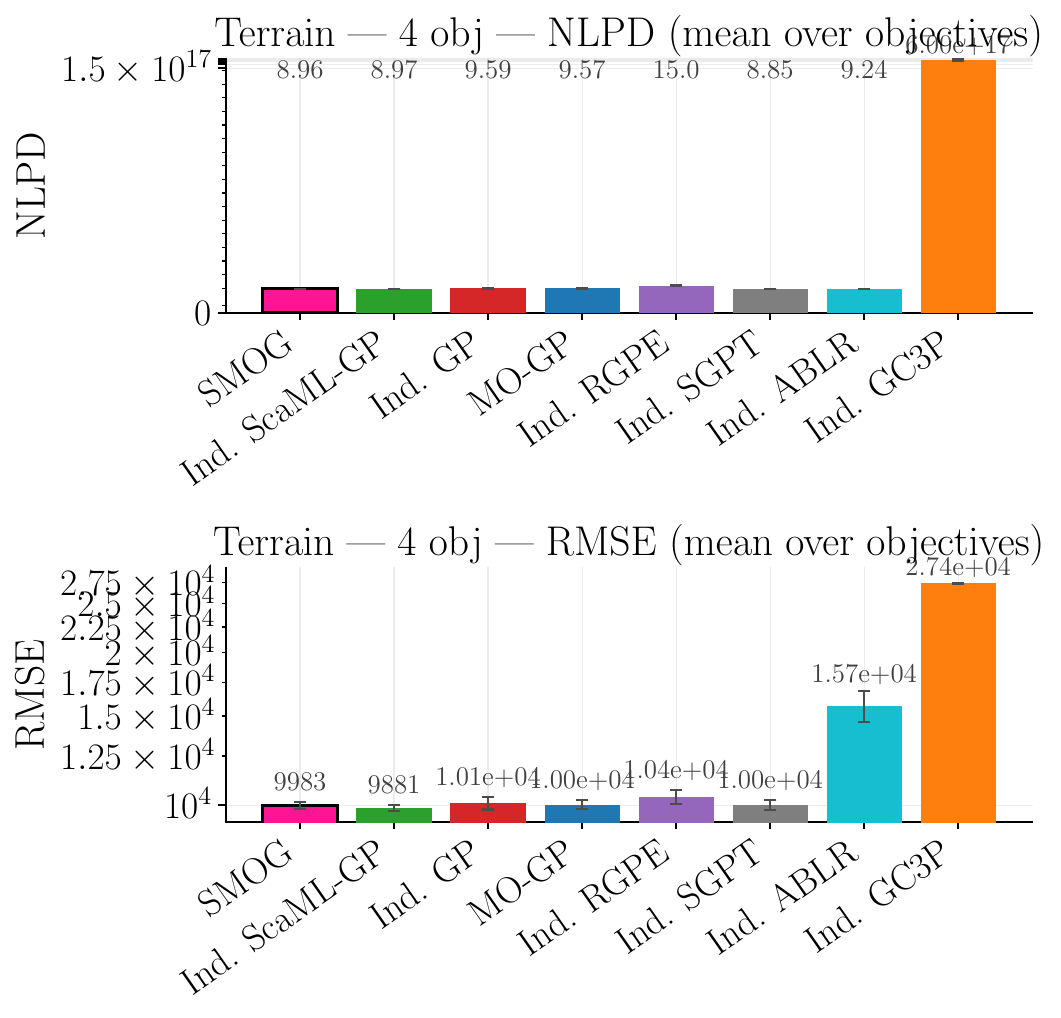}
        \caption{4 objectives}
    \end{subfigure}
    \caption{%
        Surrogate quality (NLPD and RMSE, lower is better) on the Terrain benchmark.
        Error bars denote $\pm 1$ \ac{SEM} over seeds.%
    }
    \label{fig:surrogate_quality_terrain}
\end{figure}


\end{document}